\documentclass[12pt]{amsart}
\usepackage{amsmath}
\usepackage{amssymb}

\usepackage{algorithm}
\usepackage[noend]{algpseudocode}

\usepackage{bm}
\usepackage{enumitem}
\usepackage{graphicx}
\usepackage{subcaption}
\usepackage{hyperref}

\usepackage{parskip}
\usepackage{xcolor}
\allowdisplaybreaks

\newtheorem{theorem}{Theorem}
\newtheorem{lemma}{Lemma}
\newtheorem{assumption}{Assumption}
\newtheorem{remark}{Remark}
\numberwithin{equation}{section}

\newcommand{\R}{\mathbb{R}}

\newcommand{\E}{\mathbb{E}}

\numberwithin{equation}{section}

\begin{document}
\title[SGEM: stochastic gradient with energy and momentum] 
{SGEM: stochastic gradient with energy and momentum}
\author{Hailiang Liu and Xuping Tian}
\address{Iowa State University, Mathematics Department, Ames, IA 50011} \email{hliu@iastate.edu, xupingt@iastate.edu}



\begin{abstract}
In this paper, we propose SGEM, Stochastic Gradient with Energy and Momentum, to solve a large class of general non-convex stochastic optimization problems, based on the AEGD method that originated in the work [AEGD: Adaptive Gradient Descent with Energy. arXiv: 2010.05109]. SGEM incorporates both energy and momentum at the same time so as to inherit their dual advantages. We show that SGEM features an unconditional energy stability property, and derive energy-dependent convergence rates in the general nonconvex stochastic setting, as well as a regret bound in the online convex setting. A lower threshold for the energy variable is also provided. Our experimental results show that SGEM converges faster than AEGD and generalizes better or at least as well as SGDM in training some deep neural networks.
\end{abstract}

\maketitle


\section{Introduction}\label{sec1}

In this paper, we propose SGEM: Stochastic Gradient with Energy and Momentum, to solve the following general non-convex stochastic optimization problem
\begin{equation}\label{opt}
\min_{\theta\in\R^d} f(\theta):=\E_\xi[f(\theta;\xi)],
\end{equation}
where $\E_\xi[\cdot]$ denotes the expectation with respect to the random variable $\xi$. We assume that $f$ is differentiable and bounded from below, i.e., $f^*=\inf_{\theta \in \R^d} f(\theta)>-c$ for some $c>0$.

Problem (\ref{opt}) arises in many statistical learning and deep learning models \cite{LBH15, Go16, BC18}. For such large scale problems, it would be too expensive to compute the full gradient $\nabla f(\theta)$. 
One approach to handle this difficulty is to use an unbiased estimator of $\nabla f(\theta)$. Denote the stochastic gradient at the $t$-th iteration as $g_t$, the iteration of Stochastic Gradient Descent (SGD) \cite{RM51} can be described as:
$$
\theta_{t+1}=\theta_t-\eta_t g_t,
$$
where $\eta_t$ is called the learning rate. Its convergence is known to be ensured if $\eta_t$ meets the sufficient condition:
\begin{equation}\label{con}
\sum_{t=1}^\infty \eta_t=\infty, \quad   \sum_{t=1}^\infty \eta_t^2<\infty.
\end{equation}
However, vanilla SGD suffers from slow convergence due to the variance of the stochastic gradient, which is one of the major bottlenecks for practical use of SGD \cite{Bo12, SW96}. Its performance is also sensitive to the learning rate, which is tricky to tune via (\ref{con}). Different techniques have been introduced to improve the convergence and robustness of SGD, such as variance reduction \cite{DB+14, LJ+17, JZ13, SW+19}, momentum acceleration \cite{Zhu18, SM+13}, and adaptive learning rate \cite{Du11, TH12, KB17}. Among these, momentum and adaptive learning rate techniques are most economic since they require slightly more computation in each iteration. However, training with adaptive algorithms such as Adam or its variants typically generalizes worse than SGD with momentum (SGDM), even when the training performance is better \cite{WR+18}. 

The most popular momentum technique, Heavy Ball (HB) \cite{P64} has been extensively studied for stochastic optimization problems \cite{LG20, JN+18, Q99}. SGDM, also called SHB, as a combination of SGD and momentum takes the following form
\begin{equation}\label{sgdm}
m_t=\mu m_{t-1}+g_t, \; \theta_{t+1}=\theta_t-\eta_t m_t,  
\end{equation}
where $m_0=0$ and $\mu \in (0, 1)$ is the momentum factor. This helps to reduce the variance in stochastic gradients thus speeds up the convergence, and has been found to be successful in practice \cite{SM+13}. 

AEGD originated in the work \cite{LT20} is a gradient-based optimization algorithm that adjusts the learning rate by a transformed gradient $v$ and an energy variable $r$. The method includes two ingredients: the base update rule:
\begin{equation}\label{aegd0}
\theta_{t+1}=\theta_t +2\eta r_{t+1}v_t, \quad r_{t+1} = \frac{r_{t}}{1+2\eta v^2_{t}},
\end{equation} 
and the stochastic evaluation of the transformed gradient $v_t$ as 
\begin{equation}\label{vt}
v_{t}=\frac{g_{t}}{2\sqrt{f( \theta_t; \xi_t)+c}}.
\end{equation}
AEGD is unconditionally energy stable with guaranteed convergence in energy regardless of the size of the base learning rate $\eta>0$ and how $v_t$ is evaluated. This explains why the method can have a rapid initial training process as well as good generalization performance \cite{LT20}. 

{In addition, the stochastic AEGD has also shown distinct improvement in performance as observed in experiments in \cite{LT20}, where the convergence analysis is mostly restricted to the deterministic setting. Questions we want to answer here are: (i) does Stochastic AEGD converge for non-convex optimization problems as in the deterministic setting? (ii) whether adding momentum can help accelerate convergence?  In this paper, we propose and analyze a unified framework by incorporating both energy and momentum so as to inherit their dual advantages. We do so by keeping the base AEGD update rule (\ref{aegd0}), but taking 
\begin{equation}\label{vtm}
\begin{aligned}
&v_{t}=\frac{m_{t}}{2(1-\beta^t)\sqrt{f( \theta_t; \xi_t)+c}}, \\
&m_t=\beta m_{t-1}+ (1-\beta)g_t,
\end{aligned}   
\end{equation}
where $\beta\in (0, 1)$.
We call this novel method SGEM. It reduces to AEGD when $\beta=0$. 
An immediate advantage is that with such $v_t$ one can significantly reduce the oscillations observed in the AEGD in stochastic cases. Regarding the theoretical results, we attempt to develop a convergence theory for SGEM, in both stochastic nonconvex setting and online convex setting.
}

We highlight the main contributions of our work as follows:
\begin{itemize}
\item We propose a novel and simple gradient-based method SGEM which integrates both energy and momentum. 
The only hyperparameter requires tuning is the base learning rate.
\item We show the unconditional energy stability of SGEM, and provide energy-dependent convergence rates in the general stochastic nonconvex setting,  and a regret bound for the online convex framework. We also obtain a lower threshold for the energy variable. Our assumptions are natural and mild.  
\item We empirically validate the good performance of SGEM on several deep learning benchmarks. Our results show that
\begin{itemize}
\item The base learning rate requires little tuning on complex deep learning tasks.
\item Overall, SGEM is able to achieve both fast convergence and good generalization performance. Specifically, SGEM converges faster than AEGD and generalizes better or at least as well as SGDM.
\end{itemize}
\end{itemize}

{\bf Related works.}  
The  essential  idea  behind AEGD is the  so  called  Invariant Energy Quadratizaton (IEQ) strategy, originally introduced for developing linear and unconditionally energy stable schemes for gradient flows in the form of partial differential equations \cite{Y16, ZWY17}. 
As for gradient-based methods, there has appeared numerous works on the analysis of convergence rates. In online convex setting, a regret bound for SGD is derived in \cite{Zin03}; the classical convergence results of SGD in stochastic nonconvex setting can be found in \cite{BC18}; For SGDM, we refer the readers to \cite{YJ19, YY18, LG20} for convergence rates on smooth nonconvex objectives. For adaptive gradient methods, most convergence analyses are restricted to online convex setting \cite{Du11, RK18, LX19}, while recent attempts, such as \cite{CL19, ZS19}, have been made to analyze the convergence in stochastic nonconvex setting. 

This paper is organized as follows. We first review AEGD in Section \ref{review}, then introduce the proposed algorithm in Section \ref{proalg}. Theoretical analysis including unconditional energy stability, convergence rates in both stochastic nonconvex setting and online convex setting are presented in Section \ref{thm}. In Section \ref{exp}, we report some experimental results on deep learning tasks.

{\bf Notation} For a vector $\theta\in\R^n$, we denote $\theta_{t,i}$ as the $i$-th element of $\theta$ at the $t$-th iteration. For vector norm, we use $\|\cdot\|$ to denote $l_2$ norm and use $\|\cdot\|_\infty$ to denote $l_\infty$ norm.
We also use $[m]$ to represent the list $\{1,...,m\}$ for any positive integer $m$. 

\section{Review of AEGD}\label{review}

Recall that for the objective function $f$, we assume that $f$ is differentiable and bounded from below, i.e., $f(\theta)>-c$ for some $c>0$. The key idea of AEGD introduced in \cite{LT20} is the use of an auxiliary energy variable $r$ such that 
\begin{equation}\label{frv}
 \nabla f(\theta) = 2rv, \quad  
v:=\nabla \sqrt{f(\theta)+c}, 
\end{equation} 
where $r$, taking as $\sqrt{f(\theta)+c}$ initially, will be updated together with $\theta$, and $v$ is dubbed as the transformed gradient. The gradient flow $\dot \theta=-\nabla f(\theta)$ is then replaced by 
$$
\dot \theta=-2rv, \quad \dot r= v\cdot \dot \theta. 
$$
A simple implicit-explicit discretization gives the following AEGD update rule:
\begin{subequations}\label{aegd-g}
\begin{align}
&v_t=\frac{\nabla f(\theta_t)}{2\sqrt{f(\theta_t)+c}}.\\
&\theta_{t+1}=\theta_t-2\eta r_{t+1}v_t, \\ &r_{t+1}-r_t=v_t \cdot (\theta_{t+1}-\theta_t).   
\end{align}    
\end{subequations}
This yields a decoupled update for $r$ as 
$r_{t+1}=r_t/(1+2\eta |v_t|^2)$, which serves to adapt the learning rate.
For large-scale problems, stochastic sampling approach is preferred. Let $f(\theta_t;\xi_t)$ be a stochastic estimator of the function value $f(\theta_t)$ at the $t$-th iteration, $g_t$ be a stochastic estimator of the gradient $\nabla f(\theta_t)$, then the stochastic version of AEGD is still (\ref{aegd-g}) but with $v_t$ replaced by 
$$
v_t=\frac{g_t}{2\sqrt{f(\theta_t;\xi_t)+c}}.
$$
Usually, $g_t$ should be required to satisfy 
$\E[g_t]=\nabla f(\theta_t)$ and $\E[\|g_t\|^2]$ bounded.  
Correspondingly, an element-wise version of AEGD for stochastic training reads as
\begin{subequations}\label{aegd}
\begin{align}
& v_{t,i} = \frac{g_{t,i}}{2\sqrt{f(\theta_t;\xi_t)+c}},\quad i\in[n], \\
& r_{t+1,i} =\frac{r_{t,i}}{1+2\eta v_{t,i}^2},\quad r_{1,i}=\sqrt{f(\theta_1;\xi_1)+c}, \\
& \theta_{t+1,i}=\theta_{t,i}-2\eta r_{t+1,i}v_{t,i}.
\end{align}
\end{subequations}
The element-wise AEGD allows for different effective learning rates for different coordinates, which has been empirically verified to be more effective than the global AEGD (\ref{aegd-g}). For further details, we refer to \cite{LT20}. We will focus only on the element-wise version of SGEM in what follows.

\section{AEGD with momentum}\label{proalg}
In this section, we present a novel algorithm to improve AEGD with added momentum in the following manner: 
\begin{subequations}\label{mv}
\begin{align}
m_{t} &= \beta m_{t-1}+(1-\beta)g_{t},\quad m_{0}=\bf{0},\\
v_{t} &= \frac{m_{t}}{2(1-\beta^t)\sqrt{f(\theta_t;\xi_t)+c}},
\end{align}
\end{subequations}
where $\beta\in(0,1)$ controls the weight for gradient at each step. With $v_t$ so defined, the update rule for $r$ and $\theta$ are kept the same as given in (\ref{aegd}b, c). The relation between the energy and the momentum in the algorithm is realized through relating $m_t$ ( as an approximation to $\nabla f$)  to $v_t$ (as an approximation of $\nabla F = \frac{\nabla f}{2\sqrt{f+c}}$), where $v_t$ is used to update the energy $r_{t+1}$. In machine learning tasks,  $f$ as a loss function is often in the form of $f(\theta)=\frac{1}{m}\sum_{i=1}^{m}l_i(\theta)$, where $l_i$,  measuring the distance between the model output and target label at the $i$-th data point, is typically bounded from below, that is, $l_i(\theta)>-c, \forall i\in[m]$, for some $c>0$. 
Hence $c$ in (\ref{mv}b) can be easily chosen in advance so that $f(\theta_t;\xi_t)$ as a random sample from $\{l_i(\theta_t)\}_{i=1}^{m}$ is bounded below by $-c$ for all $t\in[T]$. 
We summarize this in Algorithm \ref{alg} (called SGEM, for short).

A key feature of SGEM is that it incorporates momentum into AEGD without changing the overall structure of the AEGD algorithm (the update of $r$ and $\theta$ remain the same) so that it is shown (in Section \ref{thm}) to still enjoy the unconditional energy stability property as AEGD does. In addition, by using $m_t$ instead of $g_t$, the variance can be significantly reduced. In fact, as proved in \cite{LG20}, under the assumption $\E_{\xi_t}[\|g_t-\nabla f(\theta_t)\|^2]= \sigma^2_g<\infty$,  $m_t$, which can be expressed as a linear combination of the gradients at all previous steps, 
\begin{equation}\label{mmt}
m_{t}=(1-\beta)\sum_{j=1}^{t}\beta^{t-j}g_j,     
\end{equation}
enjoys a reduced ``variance" in the sense that
$$
\E_{\xi_t}\Bigg[\bigg\|m_t-(1-\beta)\sum_{j=1}^{t}\beta^{t-j}\nabla f(\theta_j)\bigg\|^2\Bigg]\leq (1-\beta)\sigma^2_g.
$$
{Hence compared with AEGD, we expect SGEM to achieve faster convergence due to the reduced variance. This is also the main advantage of SGEM over AEGD that we observed in experiments (See Figure \ref{fig:mnist}).}

\begin{figure}[ht]
\centering
\includegraphics[width=0.5\linewidth]{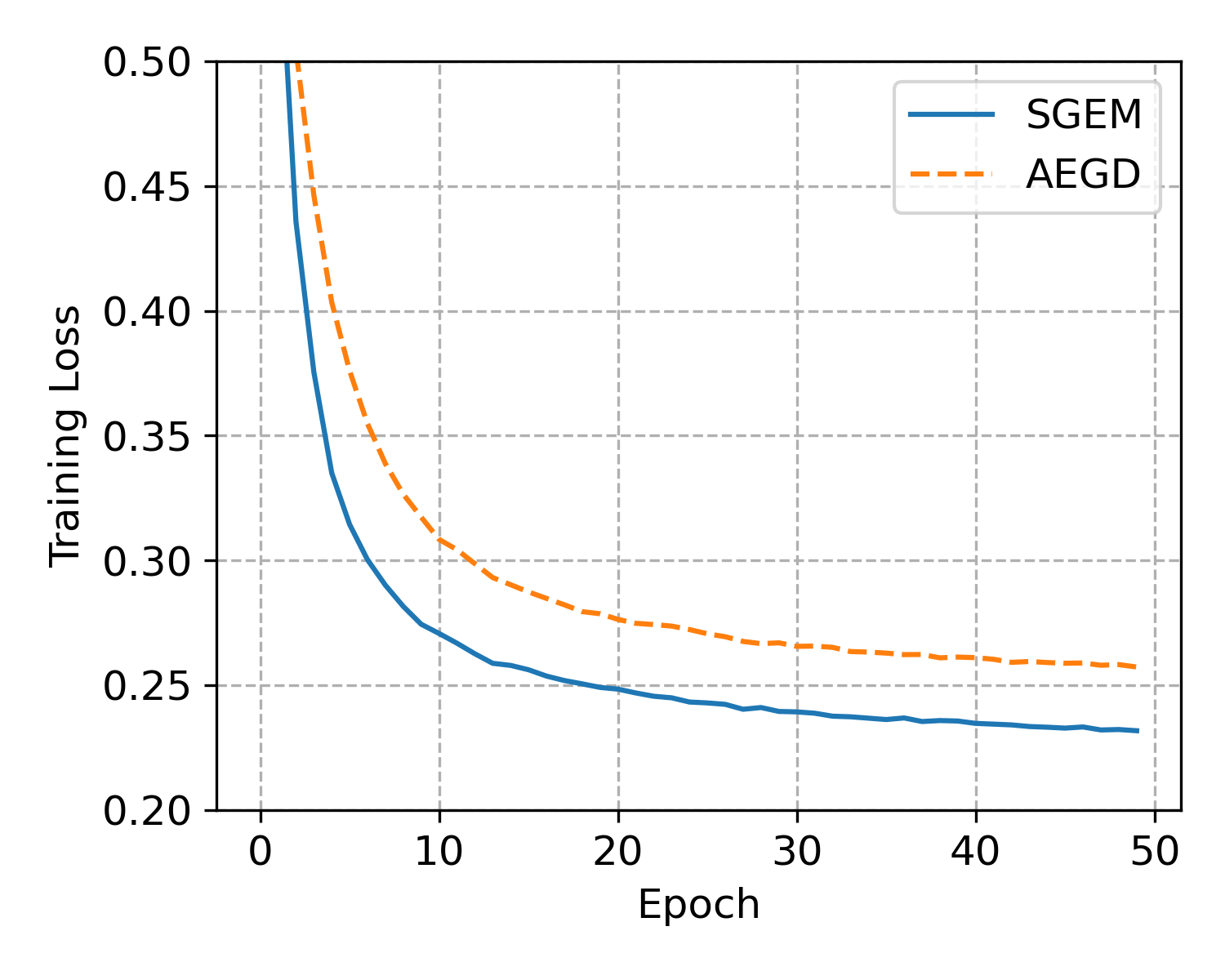}
\caption{The training loss of a Logistic regression model for the multiclass classification problem on MNIST dataset. Unlike the common setting where the batch size is set as 128 or 64, we take the batch size as 1 so that the problem of the variance of the stochastic gradient is more severe.}
\label{fig:mnist}
\end{figure}%

\begin{algorithm}
\caption{SGEM. Good default setting for parameters are $\eta=0.2$, $\beta=0.9$}
\label{alg}
\begin{algorithmic}[1] 
\State 
Require: the base learning rate $\eta$; a constant $c$ such that $f(\theta_t;\xi_t)+c>0$ for all $t\in[T]$; a momentum factor $\beta \in (0, 1)$.
\State Initialization: $\theta_1$; $m_0=\bf{0}$; $r_1=\sqrt{f(\theta_1;\xi_1)+c}\,\bf{1}$
\For{$t=1$ to $T-1$}
\State Compute gradient: $g_t=\nabla f(\theta_t;\xi_t)$
\State $m_{t}=\beta m_{t-1}+ (1-\beta) g_t$ (momentum update)
\State $v_t=  m_{t}/(2(1-\beta^t)\sqrt{f(\theta_t;\xi_t)+c})$ (transformed momentum)
\State $r_{t+1}=r_{t}/(1+2\eta v_t\odot v_t)$ (energy update)
\State $\theta_{t+1}=\theta_t- 2\eta r_{t+1}\odot v_{t}$ (state update)
\EndFor
\State {\bfseries return} ${\theta}_T$
\end{algorithmic}
\end{algorithm}

\begin{remark}
(i) In Algorithm \ref{alg}, we use $x\odot y$ to denote element-wise product, $x/y$ to denote element-wise division of two vectors $x,y\in \R^n$.\\
(ii) It is clear that $m_t$ defined in (\ref{mmt}) is not a convex combination of $g_j$, this is why there is a factor  $1-\beta^t$ in  (\ref{mv}b); such treatment is dubbed as bias correction in \cite{KB17} for Adam. \\
(iii) In most machine learning problems, we have $f(\theta)\geq 0$, for which a good default value for $c$ in Algorithm \ref{alg} is $1$. 
\end{remark}

\section{Theoretical results}\label{thm}
In this section, we present our theoretical results, including the unconditional energy stability of SGEM,  the convergence of SGEM for the general stochastic nonconvex optimization, 
a lower bound for energy $r_T$, and a regret bound in the online convex setting.

\subsection{Unconditional energy stability}

\begin{theorem}(Unconditional energy stability)
\label{thm1} SGEM in Algorithm \ref{alg}
is unconditionally energy stable in the sense that for any step size $\eta>0, i\in [n]$, 
\begin{equation}\label{srei+}
\begin{aligned}
   \E[r^2_{t+1,i}]=\E[r^2_{t,i}] &-\E[(r_{t+1,i} -r_{t,i})^2]\\
   &- \eta^{-1} \E[(\theta_{t+1,i}-\theta_{t,i})^2],
   \end{aligned}    
\end{equation}
that is $\E[r_{t,i}]$ is strictly decreasing and convergent with $\E[r_{t,i}] \to \E[r_i^*]$ as $t\to \infty$, and also
\begin{equation}\label{srev1+}
\begin{aligned}
   &\lim_{t\to \infty}\E [(\theta_{t+1,i}-\theta_{t,i})^2]=0,\\ &\sum_{t=1}^\infty\E[(\theta_{t+1,i}-\theta_{t,i})^2] \leq \eta (f(\theta_1)+c).
\end{aligned}    
\end{equation}

\end{theorem}

\begin{remark}\label{rem4.1}
(i) The unconditional energy stability only depends on (\ref{aegd}b, c), irrespective of the choice for $v_t$. This property essentially means that the energy variable $r_t$, which serves to approximate $\sqrt{f(\theta_t)+c}$, is strictly decreasing for any $\eta>0$.
(ii) (\ref{srev1+}) indicates that the sequence $\|\theta_{t+1}-\theta_t\|$ converges to zero at a rate of at least $1/\sqrt{t}$. We note that this does not guarantee the convergence of $\{\theta_t\}$ unless additional information on the geometry of $f$ is available.
\end{remark}
\begin{proof}From (\ref{aegd}b, c) we have 
\begin{align*}
(\theta_{t+1,i}-\theta_{t,i})^2
&=4\eta^2 r^2_{t+1,i}v^2_{t,i}\quad\text{(By \ref{aegd}c)}\\
&=(2\eta r_{t+1,i})(r_{t,i}-r_{t+1,i})\quad\text{(By \ref{aegd}b)}\\
&=\eta((r^2_{t,i}-r^2_{t+1,i})-(r_{t,i}-r_{t+1,i})^2).
\end{align*}
This upon taking expectation ensures the asserted properties. Such proof with no use of the special form of $v_t$, is the same as that for AEGD (see \cite{LT20}).  
\end{proof}

\subsection{Convergence analysis}

Below, we state the necessary assumptions that are commonly used for analyzing the convergence of a stochastic algorithm for nonconvex problems, and notations that will be used in our analysis.
\begin{assumption}\label{asp}
\begin{enumerate}\setlength{\itemindent}{-1em}
\itemsep-0.3em 
\item ({\bf Smoothness}) The objective function in (\ref{opt}) is $L$-smooth: for any $x,y\in\R^n$, 
$$
f(y)\leq f(x)+\nabla f(x)^\top(y-x)+\frac{L}{2}\|y-x\|^2.
$$
\item ({\bf Independent samples}) The random samples $\{\xi_t\}_{t=1}^\infty$ are independent.
\item ({\bf Unbiasedness}) The estimator of the gradient and function value are unbiased: 
$$
\E_{\xi_t}[g_t]=\nabla f(\theta_t),\quad \E_{\xi_t}[f(\theta_t;\xi_t)]= f(\theta_t).
$$
\end{enumerate}
\end{assumption}
Denote the variance of the stochastic gradient and function value by $\sigma_g$ and $\sigma_f$, respectively:
$$
\E_{\xi_t}[\|g_t-\nabla f(\theta_t)\|^2]= \sigma^2_g,\quad \E_{\xi_t}[\lvert f(\theta_t;\xi_t)- f(\theta_t)\rvert^2]= \sigma^2_f.
$$
Before presenting our energy-dependent convergence rates, we first state a result on the lower bound of $r_T$.  
\begin{theorem}[Lower bound of $r_T$]
\label{thm4}
Under Assumption \ref{asp} and assume that the stochastic gradient and function value are bounded such that $\|g_t\|_\infty\leq G_\infty$ and $0<a\leq f(\theta_t;\xi_t)+c\leq B$, in the absence of noise, 
we have 
\begin{equation}\label{minr+}
\min_i r_{T,i}\geq \max\{F(\theta^*)-\eta D_1-\beta D_2,0\},   
\end{equation}
where $L_F$ is given in (\ref{LF}) and
\begin{align*}
&D_1 = \frac{L_F n F^2(\theta_1)}{2}, \quad D_2 =\frac{\sqrt{B}nF(\theta_1)}{(1-\beta)\sqrt{a}}.
\end{align*}
This implies 
\begin{equation}\label{aa}
\min_i r_{T,i}>\min_i r^*_i>0 \quad\text{if}\quad \eta D_1 +\beta D_2 <F(\theta^*).
\end{equation}
\end{theorem}

\begin{remark}
(i) (\ref{aa}) is only a sufficient condition, not intended as a guide for choosing $\eta$. We observe from our experimental results that the upper bound for $\eta$ to guarantee the positiveness of $r^*_i$ can be much larger (See Figure \ref{fig:minr}).\\
(ii) 
In Theorem \ref{thm4}, we measure how far $r^*$ can deviate from $F(\theta^*)$ in the worst scenario using simple error split, where $\eta D_1$ is the error brought by the step size $\eta$, and $\beta D_2$ is due to the use of momentum. 
\end{remark}

We only present a sketch of proofs for Theorem \ref{thm4} and \ref{thm2} here, using notation $\tilde F_t=\sqrt{f(\theta_t;\xi_t)+c}$, $\eta_t=\eta/\tilde F_t$ and viewing $r_{t+1}$ as a $n\times n$ diagonal matrix that is made up of $[r_{t+1,1},...,r_{t+1,i},...,r_{t+1,n}]$. Detailed proofs, including two crucial lemmas and the full proof for Theorem \ref{thm3}, are deferred to the appendix. 

\begin{proof}
Using the $L_F$-smoothness of $F(\theta)$, we have
\begin{equation}\label{LFe}
F(\theta_{t+1}) -  F(\theta_t)
\leq\nabla F(\theta_t)^\top (\theta_{t+1}-\theta_t) +\frac{L_F}{2}\|\theta_{t+1}-\theta_t\|^2,
\end{equation}
in which the key term $\nabla F(\theta_t)^\top (\theta_{t+1}-\theta_t)$ can be decomposed into three terms:
\begin{gather*}
(\nabla F(\theta_t)-\frac{g_t}{2\tilde F_t})^\top(\theta_{t+1}-\theta_t),\\ (\frac{g_t}{2\tilde F_t}-\frac{v_t}{1-\beta})^\top(\theta_{t+1}-\theta_t),\quad (\frac{v_t}{1-\beta})^\top(\theta_{t+1}-\theta_t).
\end{gather*}
The first two terms are bounded by using the bounded variance assumption and (\ref{srev1+}), respectively. We convert the last term, using (recall \ref{aegd-g}c)
$$
r_{t+1,i}-r_{t,i}=v_{t,i}(\theta_{t+1,i}-\theta_{t,i}), 
$$
into expressions in terms of $r_{t+1,i}-r_{t,i}$, which upon summation is bounded by $r_{T,i}$. The last term in (\ref{LFe}) is bounded again by using (\ref{srev1+}).
\end{proof}

We proceed to discuss the convergence results. First note that the $L$-smoothness of $f(\theta)$ implies the $L_F$-smoothness of $F(\theta)=\sqrt{f(\theta)+c}$ with
\begin{equation}\label{LF}
L_F=
\frac{1}{2F(\theta^*)} \left( L+ \frac{G^2_\infty}{2F^2(\theta^*)}\right).     
\end{equation}
This will be used in the following result and its proof. 

\begin{theorem}\label{thm2} Let $\{\theta_t\}$ be the solution sequence generated by Algorithm \ref{alg} with a fixed $\eta>0$. Under the same assumptions as in Theorem \ref{thm4}, we have $\sigma_g\leq G_\infty$ and for all $T\geq 1$ and $\epsilon \in(0, 1)$,
\begin{align*}
&\quad\;\E\Bigg[\min_{1\leq t \leq T}\|\nabla f(\theta_t)\|^{2-2\epsilon} \Bigg]\\
& \leq \left( \frac{C_1+C_2n+C_3\sigma_g \sqrt{ nT}}{\eta T } \right)^{1-\epsilon}  \cdot 
\E\Big[(\min_i r_{T,i})^{1-1/\epsilon}\Big]^{\epsilon}, 
\end{align*}
where $C_1, C_2, C_3$ are constants depending on $\beta, \eta, L, G_\infty, a, B, n$ and  $f(\theta_1)+c$. 
\end{theorem}

\begin{remark}
{(i) The stated bound is a hybrid estimate depending on both $T$ and $r_T$, which can be helpful in adjusting the base step size depending on the asymptotic behavior of $r_T$. } 
(ii) In Theorem \ref{thm4} below, we identify a sufficient condition for ensuring a lower threshold for $r_{T, i}$. Numerically, we observe that for reasonable choice of $\eta$, $r_{T,i}$ either stays above a positive threshold or decays but much slower than $1/\sqrt{T}$ (See Figures \ref{fig:minr}); in either case the result in Theorem \ref{thm2} ensures the convergence.
\begin{figure}[ht]
\centering
\includegraphics[width=0.5\linewidth]{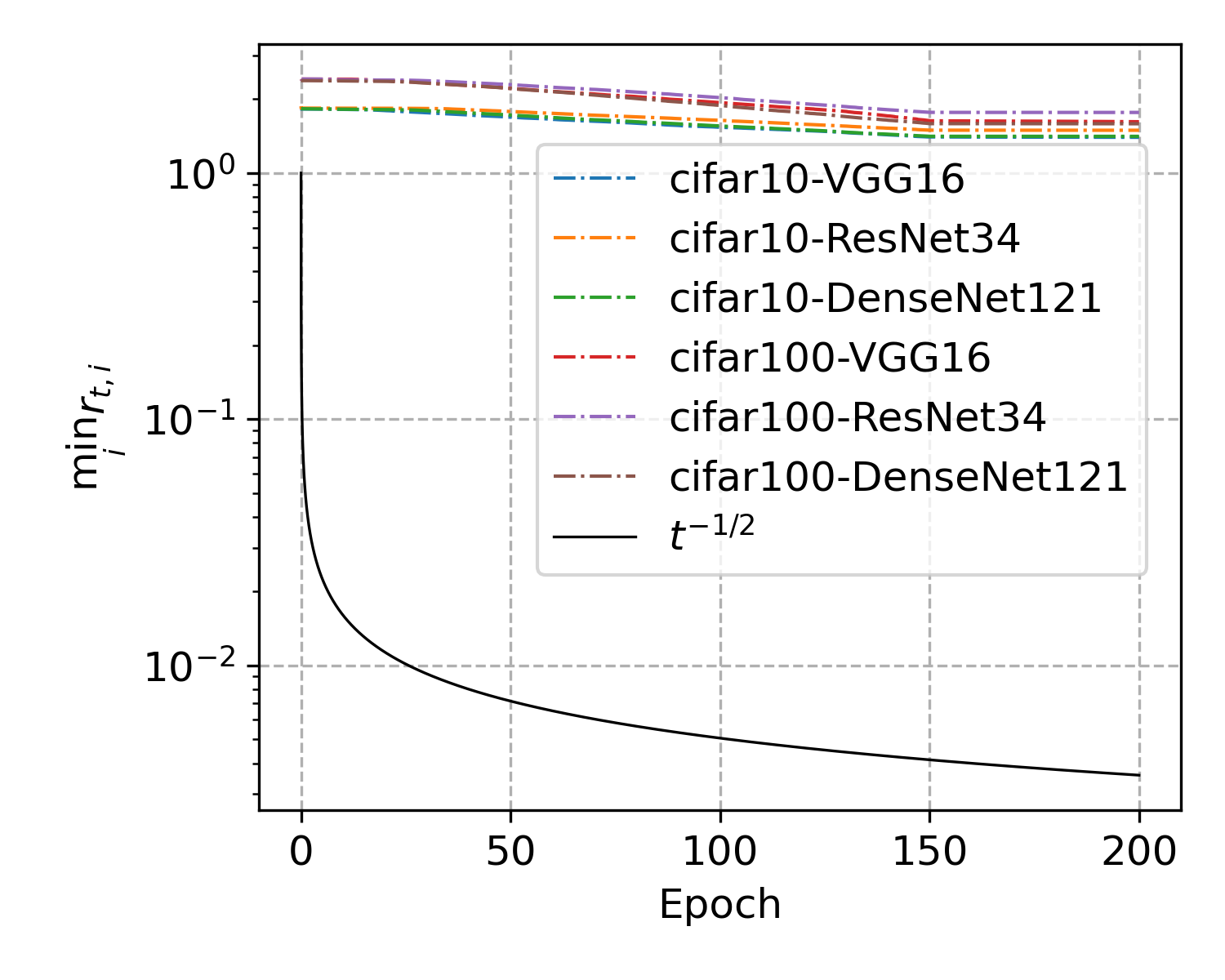}
\caption{$\min_i r_{t,i}$ of SGEM with default base learning rate $0.2$ in training DL tasks.}
\label{fig:minr}
\end{figure}%
(iii) The assumption that the magnitude of the stochastic gradient is bounded is standard in nonconvex stochastic analysis \cite{BC18}. The upper bound on the stochastic function value is technically needed to bound $m_t$ since $m_t = 2(1-\beta^t)\sqrt{f+c}v_t$. To bound $v_t$, we don't need an upper bound on $f$.
\end{remark}

\begin{proof} 
Using the $L$-smoothness of $f$, we have
\begin{equation}\label{Lfe}
f(\theta_{t+1})- f(\theta_t) \leq \nabla f(\theta_t)^\top(\theta_{t+1}-\theta_t)+\frac{L}{2}\|\theta_{t+1}-\theta_t\|^2. 
\end{equation}
The first term on the RHS is carefully regrouped as

\begin{align*}
-\frac{1-\beta}{1-\beta^t}\nabla f(\theta_t)^\top \eta_{t-1} r_{t}g_{t} + \frac{1-\beta}{1-\beta^t}\nabla f(\theta_t)^\top (\eta_{t-1} r_{t}-\eta_t r_{t+1}) g_{t}
-\frac{\beta}{1-\beta^t}\nabla f(\theta_t)^\top \eta_t r_{t+1}m_{t-1}.
\end{align*}

Taking a conditional expectation on the first term gives 
\begin{align*}
\quad\frac{1-\beta}{1-\beta^t}\eta_{t-1} \nabla f(\theta_t)^T r_t \nabla f(\theta_t)
\geq (1-\beta)\frac{\eta}{\sqrt{B}}\min_i r_{t, i}\|\nabla f(\theta_t)\|^2.    
\end{align*}
We manage to bound the other two terms in terms of  
$$
\sum_{i=1}^{n}\sum_{t=1}^{T}r_{t+1,i}g^2_{t,i}\quad\text{and}\quad \sum_{i=1}^{n}\sum_{t=1}^{T}r_{t+1,i}m^2_{t,i}.
$$
Their bounds are presented in Lemma \ref{lemrv}. The asserted bound then follows by further summation in $t$ with telescope cancellation for $f(\theta_{t+1})- f(\theta_t)$ and bounding the last term in (\ref{Lfe}) using (\ref{srev1+}).
\end{proof}

\subsection{Regret bound for Online convex optimization}
Our algorithm is also applicable to the online optimization that deals with the optimization problems having no or incomplete knowledge of the future (online). In the framework proposed in \cite{Zin03},  at each step $t$, the goal is to predict the parameter $\theta_t\in\mathcal{F}$, where $\mathcal{F}\subset\R^n$ is a feasible set, and evaluate it on a previously unknown loss function $f_t$. The nature of the sequence is unknown in advance, the SGEM algorithm needs to be modified. This can be done by replacing $f(\theta_t, \xi_t)$ by $f_t(\theta_t)$ and taking  
$
g_t =\nabla f_t(\theta_t)
$
in $v_t$ defined in (\ref{mv}), i.e., 
\begin{subequations}\label{onv} 
\begin{align}
&m_{t} = \beta m_{t-1} + (1-\beta)\nabla f_t(\theta_t),\\
&v_{t}= \frac{m_{t}}{2(1-\beta^t)\sqrt{f_t(\theta_t)+c}}.  
\end{align}    
\end{subequations}
This algorithm is also unconditionally energy stable as pointed out in Remark \ref{rem4.1}. For convergence, we evaluate our algorithm using the regret, that is the sum of all the previous difference between the online prediction $f_t(\theta_t)$ and the best fixed point parameter $f_t(\theta^*)$ from a feasible set $\mathcal{F}$: 
$$
R(T)=\sum_{t=1}^{T}[f_t(\theta_t)-f_t(\theta^*)],
$$
where $\theta^* = {\rm argmin}_{\theta\in \mathcal{F}} \sum_{t=1}^{T}f_t(\theta)$. 
For convex objectives we have the following regret bound.  
\begin{theorem}
\label{thm3} Let $\{\theta_t\}$ be the solution sequence generated by SGEM with a fixed $\eta>0$. Assume that $\|x-y\|_\infty\leq D_\infty$ for all $x,y\in\mathcal{F}$, $0<a\leq f_t(\theta_t)+c\leq B$, and $\theta_t\in\mathcal{F}$ for all $t\in[T]$. When $\mathcal{F}$ and $f_t$ are convex, SGEM achieves the following bound on the regret, for all $T\geq 1$,
\begin{equation}\label{rgt}
R(T)\leq C\sqrt{nT/\eta}\left(\sum_{i=1}^{n} \frac{1}{r_{T,i}}\right)^{1/2},
\end{equation}
where $C$ is a constant depending on $\beta, B, D_\infty$ and $f_1(\theta_1)+c$.
\end{theorem}
\begin{remark} 
(i)  
If $r_{T, i}> r^*_i>0$ as in (\ref{aa}), then $R(T)$ is of order $O(\sqrt{T})$, 
which is known the best possible bound for online convex optimization [See Section 3.2 in \cite{Ha19}], 
hence the convergence holds true in the sense that
$$
\lim_{T \to \infty} \frac{R(T)}{T}=0.
$$
(ii) The bound on $\theta_t$
is typically enforced by projection onto $\mathcal{F}$ \cite{Zin03}, with which the regret bound (\ref{rgt}) can still be proven since projection is a contraction operator \cite[Chapter 3]{Ha19}. As for the upper bound on the function value, just like we remarked for Theorem \ref{thm2}, it is technically needed to bound $m_t$.
\end{remark}

\section{Numerical experiments}\label{exp}
In this section, we compare the performance of the proposed SGEM and AEGD with several other methods, including SGDM, AdaBelief \cite{ZT+20}, AdaBound \cite{LX19}, RAdam \cite{LJ+20}, Yogi \cite{ZR+18}, and Adam \cite{KB17}, when applied to training deep neural networks. \footnote{Code is available at \url{https://github.com/txping/SGEM}.} 
We consider 
three convolutional neural network (CNN) architectures: VGG-16 \cite{SZ15}, ResNet-34 \cite{HZ16}, DenseNet-121 \cite{HL+17} on the CIFAR-100 dataset \cite{KH09}; we also conduct experiments on ImageNet dataset \cite{ILSVRC15} with the ResNet-18 architecture \cite{HZ16}. 

For experiments on CIFAR-100, we employ the fixed budget of 200 epochs and reduce the learning rates by 10 after 150 epochs. The weight decay and minibatch size are set as $5\times 10^{-4}$ and $128$ respectively. For the ImageNet tasks, we run 90 epochs and use similar learning rate decaying strategy at the 30th and 60th epoch. The weight decay and minibatch size are set as $1\times 10^{-4}$ and $256$ respectively.

In each task, we only tune the base learning rate and report the one that achieves the best final generalization performance for each method:

\begin{itemize}
\item SGEM: For CIFAR-100 tasks, we use the default parameter $\eta=0.2$; for the ImageNet task, the learning rate is set as $\eta=0.3$. 
\item SGDM, AEGD:  We search learning rate among $\{0.05, 0.1, 0.2\}$.
\item AdaBelief, AdaBound, Yogi, RAdam, Adam: We search learning rate among $\{0.0005, 0.001, 0.01\}$, other hyperparameters such as $\beta_1, \beta_2, \epsilon$ are set as the default values in their literature.
\end{itemize}


\begin{figure*}[ht]
\begin{subfigure}[b]{0.33\linewidth}
\centering
\includegraphics[width=1\linewidth]{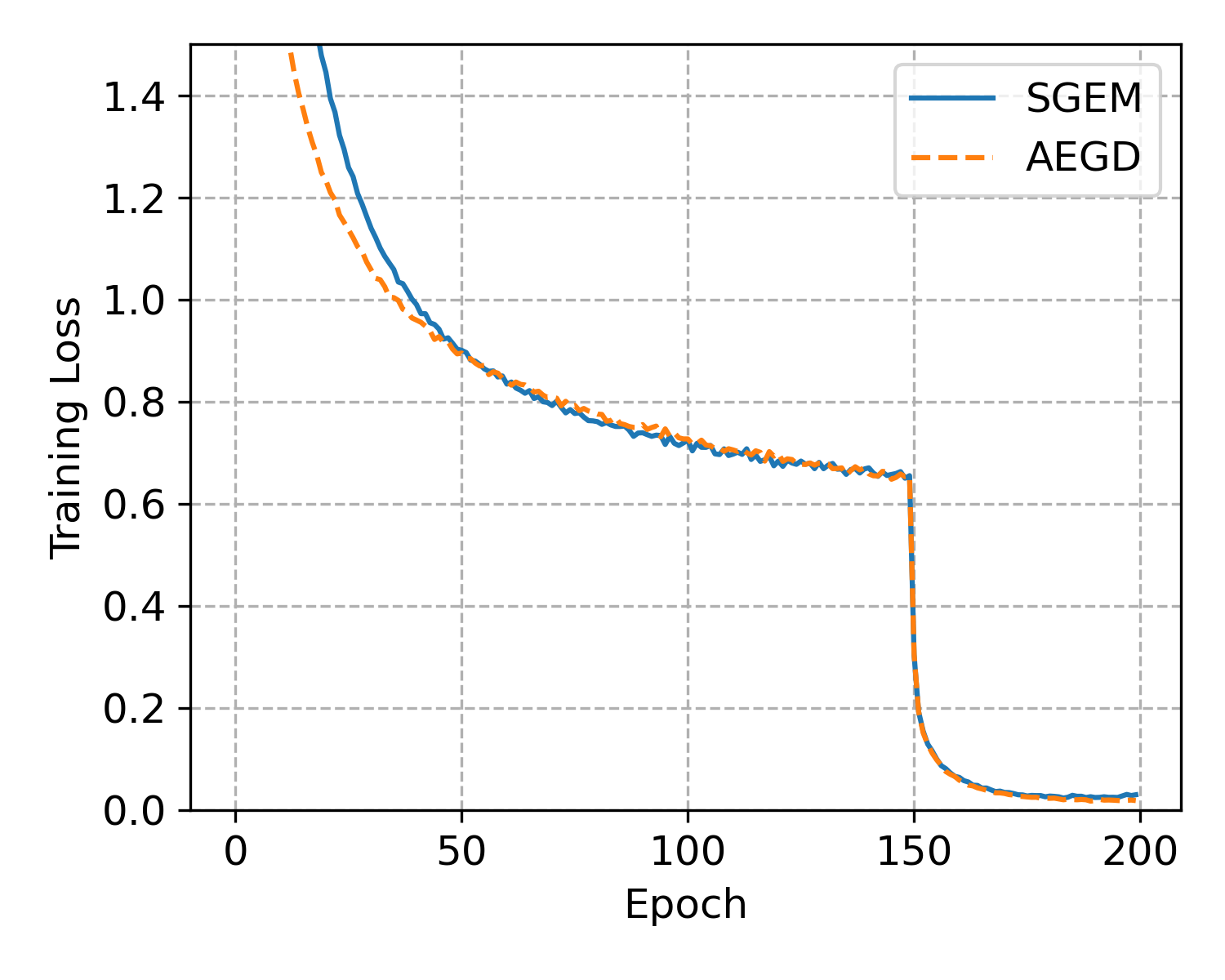}
\caption{VGG-16 training}
\end{subfigure}%
\begin{subfigure}[b]{0.33\linewidth}
\centering
\includegraphics[width=1\linewidth]{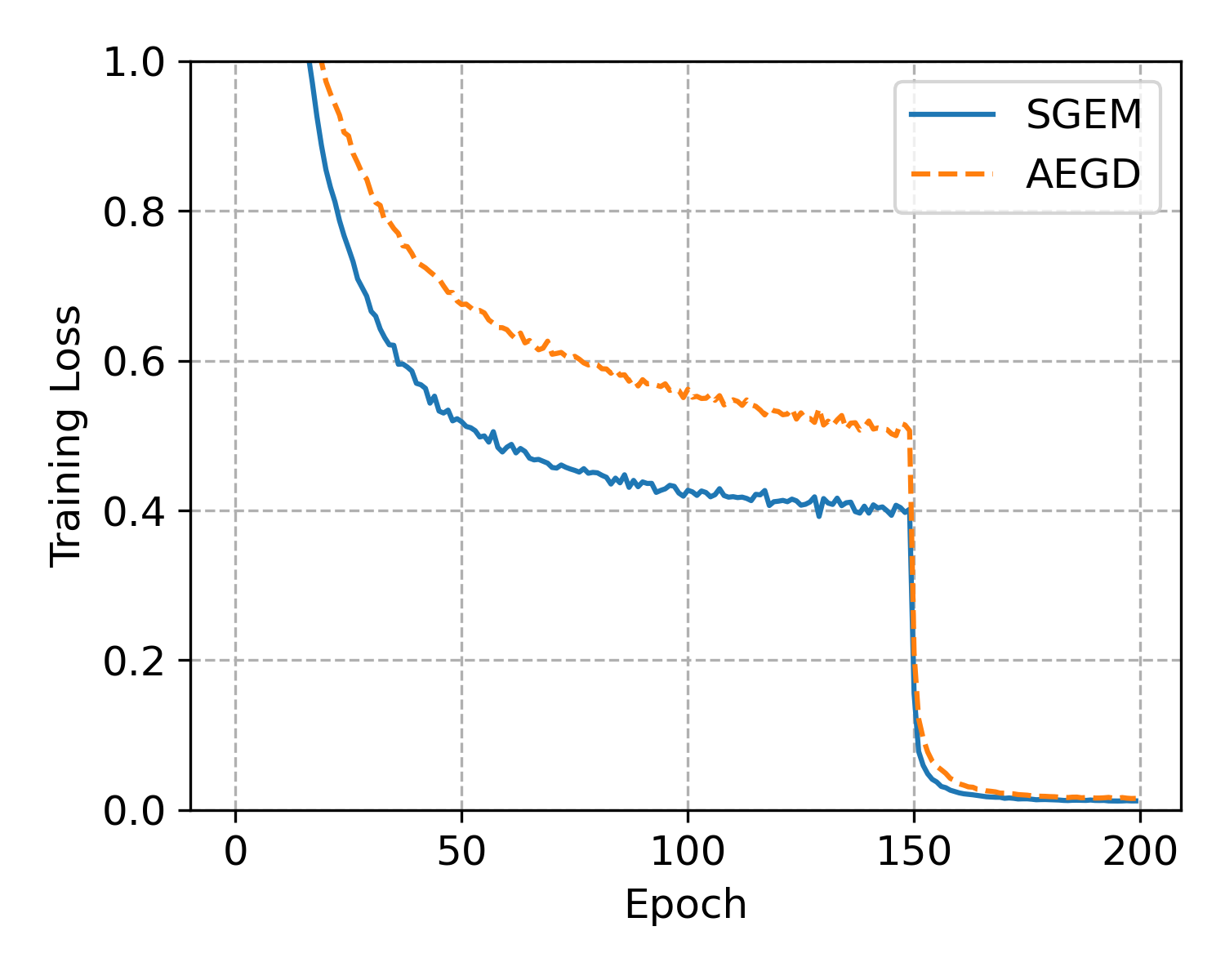}
\caption{ResNet-34 training}
\end{subfigure}%
\begin{subfigure}[b]{0.33\linewidth}
\centering
\includegraphics[width=1\linewidth]{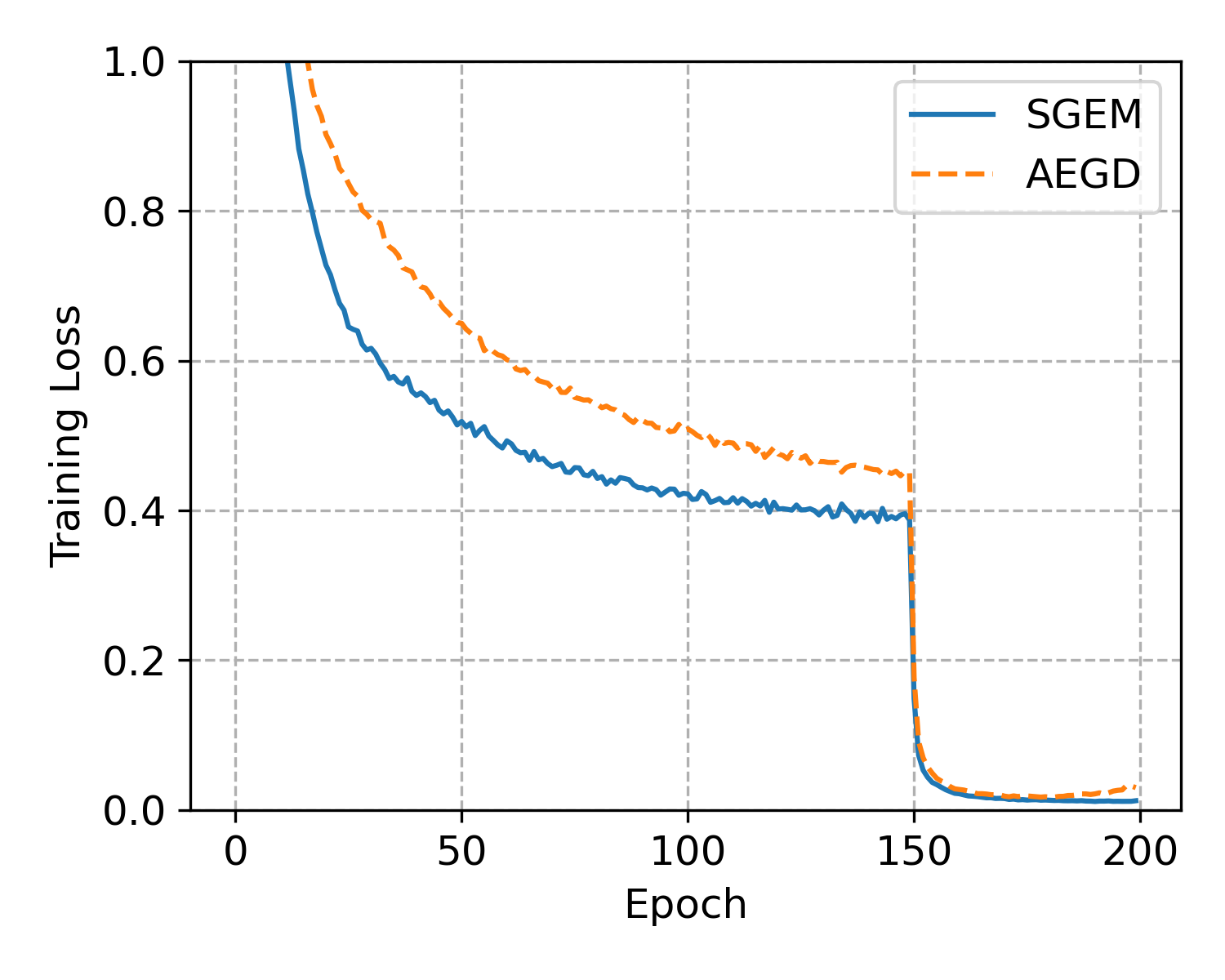}
\caption{DenseNet-121 training}
\end{subfigure}%
\newline
\begin{subfigure}[b]{0.33\linewidth}
\centering
\includegraphics[width=1\linewidth]{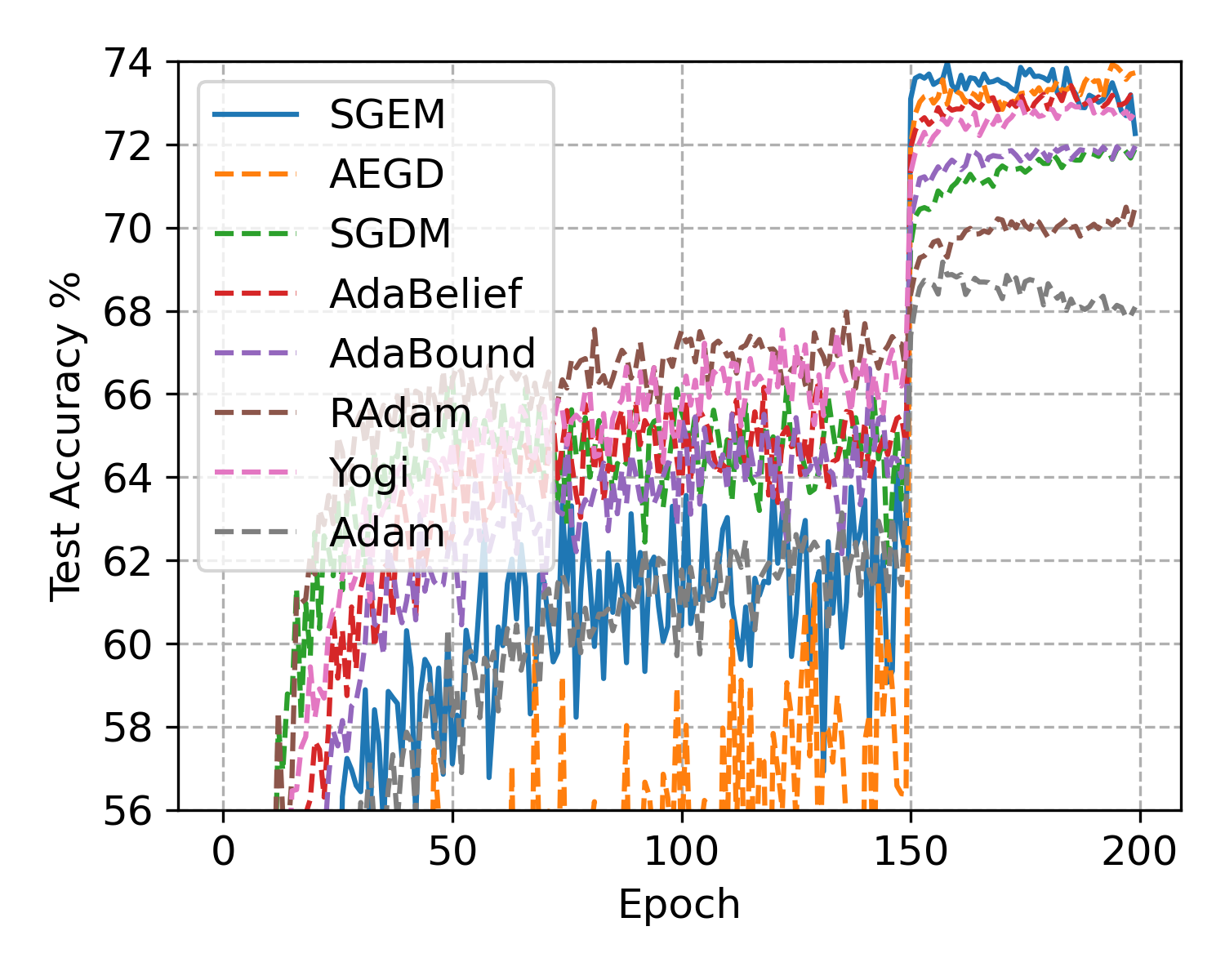}
\caption{VGG-16 test}
\end{subfigure}%
\begin{subfigure}[b]{0.33\linewidth}
\centering
\includegraphics[width=1\linewidth]{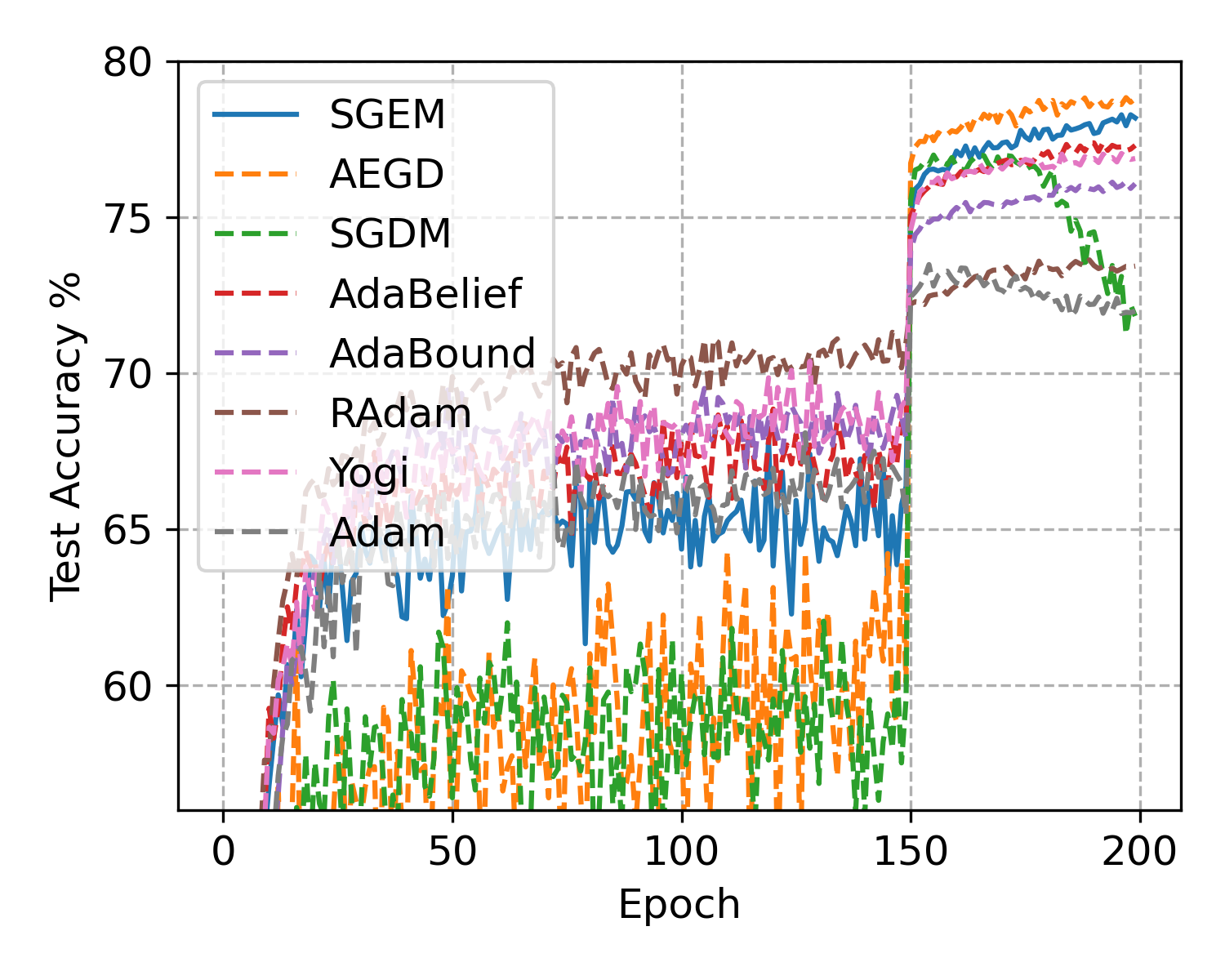}
\caption{ResNet-34 test}
\end{subfigure}%
\begin{subfigure}[b]{0.33\linewidth}
\centering
\includegraphics[width=1\linewidth]{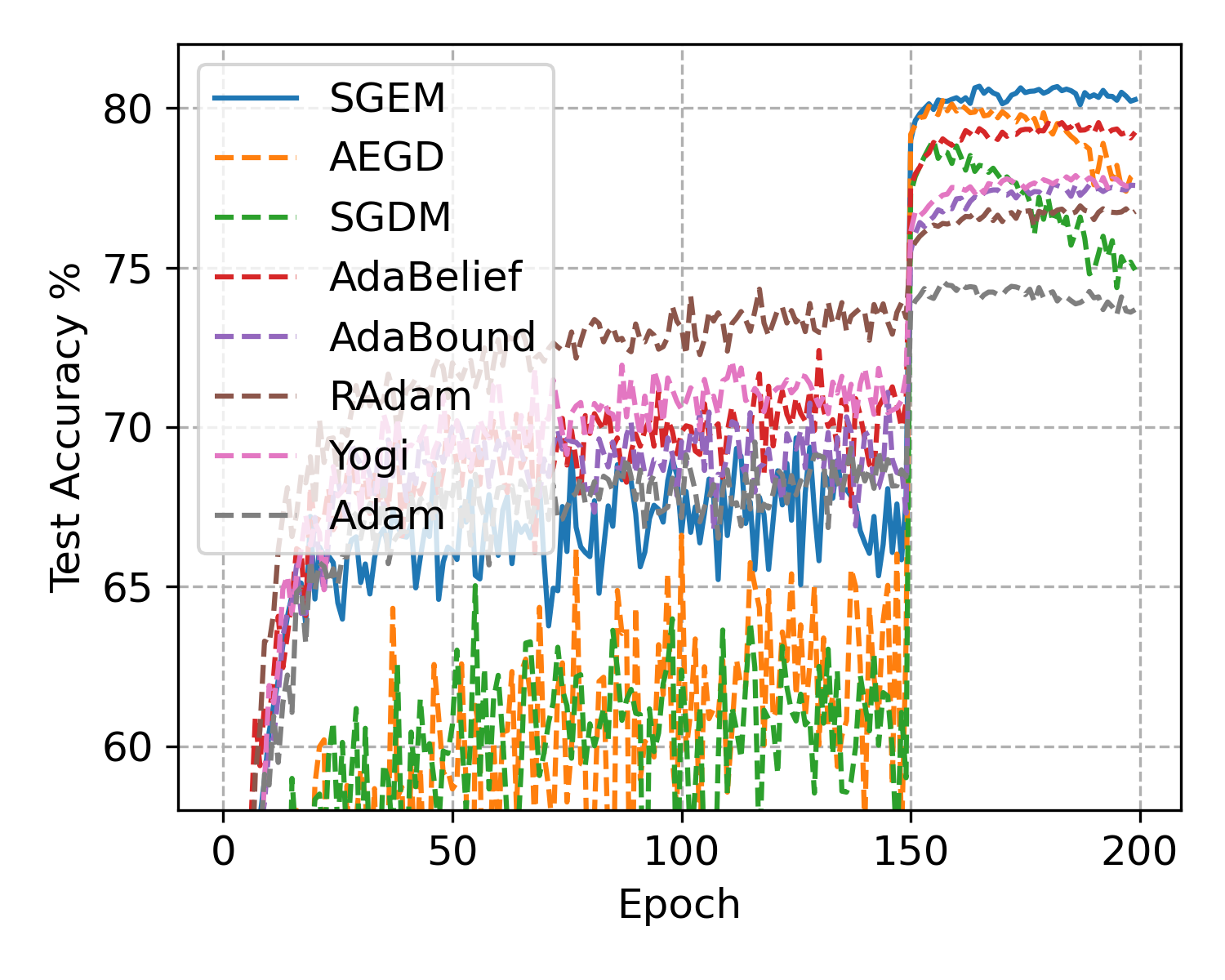}
\caption{DenseNet-121 test}
\end{subfigure}%
\captionsetup{format=hang}
\caption{Training Loss and test accuracy for VGG-16, ResNet-34 and DenseNet-121 on CIFAR-100}
\label{fig:cifar10}
\end{figure*}

\begin{figure*}[h!]
\begin{subfigure}[b]{0.5\linewidth}
\centering
\includegraphics[width=0.7\linewidth]{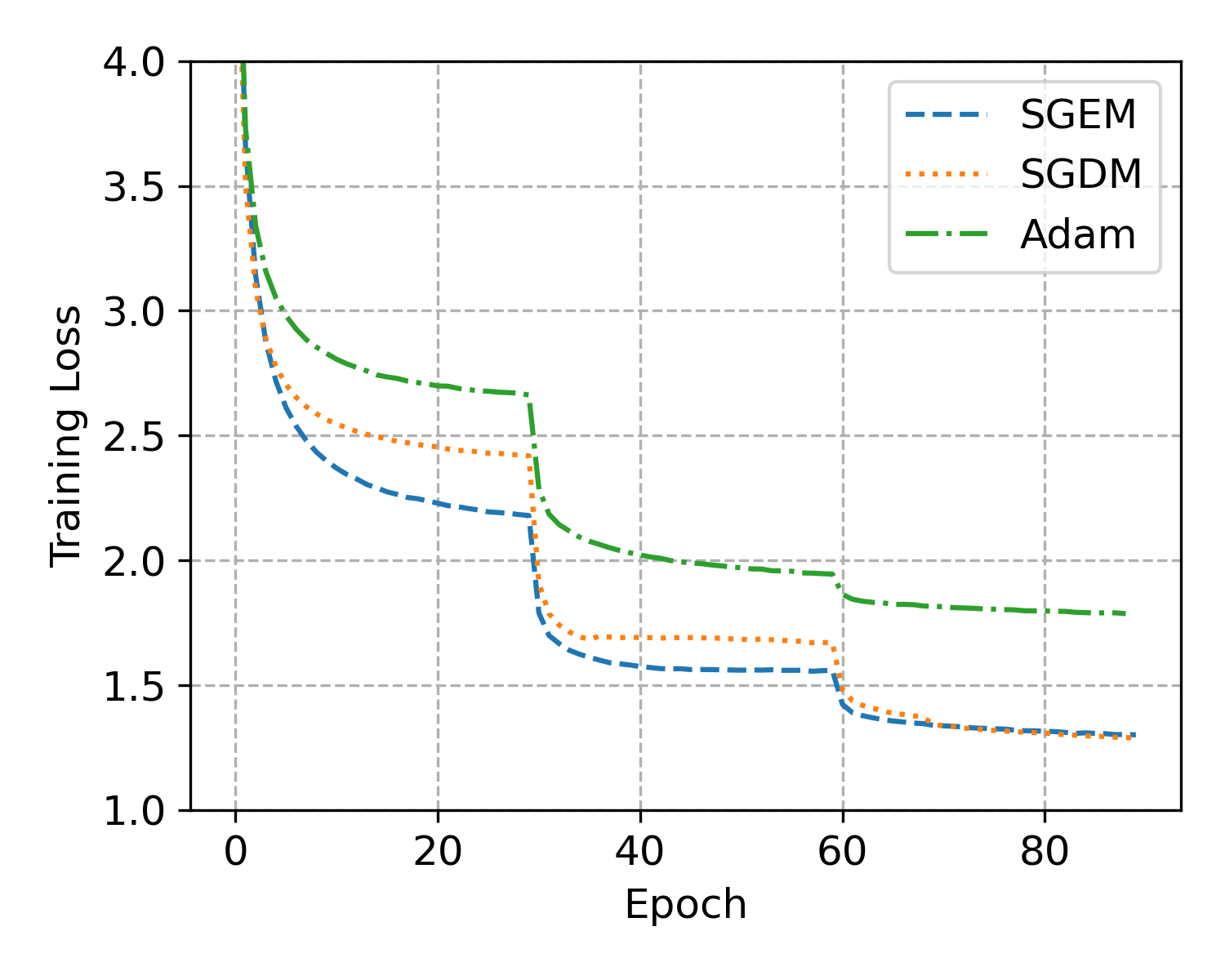}
\end{subfigure}%
\begin{subfigure}[b]{0.5\linewidth}
\centering
\includegraphics[width=0.7\linewidth]{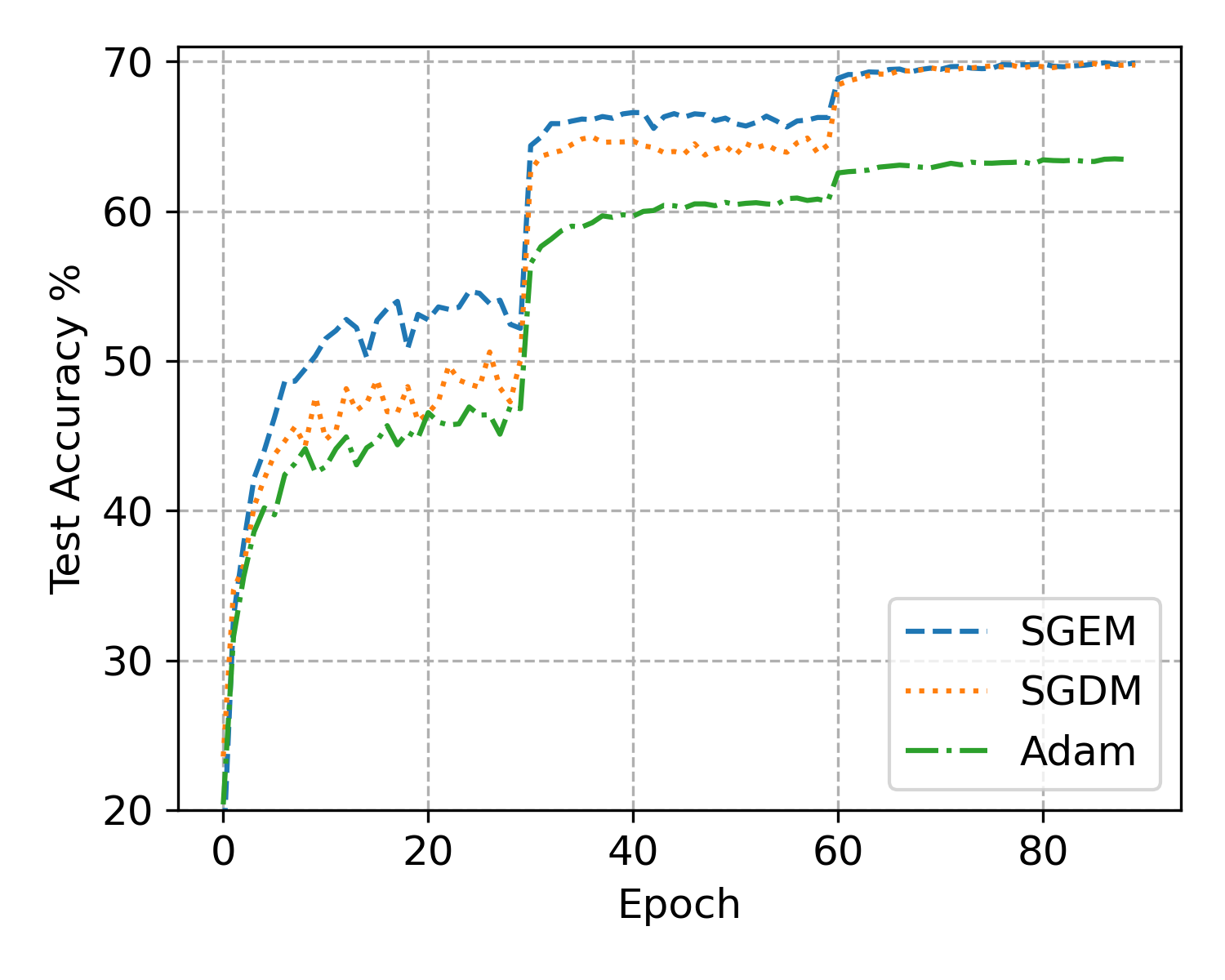}
\end{subfigure}%
\caption{Training loss and test accuracy for ResNet-18 on ImageNet}
\label{fig:imagenet}
\end{figure*}

From the experimental results of CIFAR-100, we see that in all tasks, SGEM and AEGD achieve higher test accuracy than the other methods while the oscillation of AEGD in test accuracy is significantly reduced by SGEM as expected. {To highlight the advantage of SGEM over AEGD, we also present the training loss of the two methods in each experiment, which shows SGEM indeed displays faster convergence than AEGD in most cases.}

For the ImageNet task, all existing experiments show that SGDM gives the highest test accuracy, we therefore focus only on the comparison between SGDM and SGEM, and run Adam only as a representative of other adaptive methods. The results are presented in Figure \ref{fig:imagenet}. We see that SGEM still shows faster convergence and is able to achieve comparable test accuracy as SGDM in the end of training. Here the highest test accuracy achieved by SGDM and SGEM are $69.89$ and $69.92$, respectively.

\section{Conclusion}
In this paper, we propose and analyze SGEM, which integrates AEGD with momentum. We show that SGEM still enjoys the unconditional energy stability property as AEGD, while the use of momentum helps to reduce the variance of the stochastic gradient significantly, as verified in our experiments. We also provide convergence analysis in both online convex setting and the general stochastic nonconvex setting. Since our convergence results depend on the energy variable, a lower bound on the energy is also presented. Finally, we empirically show that SGEM converges faster than AEGD and generalizes better or at least as well as SGDM on several deep learning benchmarks.

{We believe these results are important. First, we go  in the direction of establishing a theory beyond the empirical results for this new class of optimization algorithms. Second, our convergence rates can indeed provide an explanation for the good performance of AEGD algorithms.  
}

{One of the limitations of the current analysis is the fact that the obtained bound depends on $r_T$. It would be better to derive an asymptotic bound for $r_T$ as in \cite{LO19} for estimating the adaptive learning rate $\eta_T$. However, the technique used in \cite{LO19} relies on the specific form of $\eta_T$, while analysis for $r_T$ with SGEM is far from straightforward. This issue will be addressed in our future work. 
} 

Momentum is known to help accelerate gradient vectors in right directions and reduce oscillations, thus leading to faster convergence. It is desirable to quantify such effects also in the theoretical convergence bounds. Unfortunately, this has not been well understood in the literature even for SGDM. For example, SGDM (\ref{sgdm}) with constant step size is shown in \cite{YY18, YJ19} to have convergence bound: 
$$
\min_{1\leq t\leq T} \E[\|\nabla f(\theta_t)\|^2]= O\left(1/\eta T + \eta \sigma^2/(1-\beta)\right),
$$
for general smooth nonconvex objectives. An improved bound in \cite{LG20} is 
$$
O\left(1/\eta T+\eta \sigma^2 \right),
$$
since common choices for $\beta$ are close to 1. 
In contrast, when $r_T$ admits a positive lower bound, our result in Theorem \ref{thm2} indeed ensures convergence.   

Based on our observations in this paper, we list some problems for future work. First, we believe there is a threshold for $\eta^*$, such that $r_T$ either tends to a positive number or decays slower than $1/\sqrt{T}$
if $\eta<\eta^*$. This issue merits a further theoretical investigation. Second, since $r_t$ is strictly decreasing, there is a room to limit $r_t$ for controlling its decay whenever necessary. A proper energy limiter can be of help.

\appendix

\section{Proof of Theorem \ref{thm4}}
For the proofs of Theorem \ref{thm4} and Theorem \ref{thm2}, we introduce notation 
\begin{equation}\label{Ft2}
\tilde F_t:=\sqrt{f(\theta_t; \xi_t)+c}.    
\end{equation}
The initial data for $r_i$ is taken as $r_{1, i}=\tilde F_1$. We also denote the update rule presented in Algorithm \ref{alg} as
\begin{equation}\label{theta}
\theta_{t+1}=\theta_{t}-2\eta r_{t+1} v_{t},
\end{equation}
where $r_{t+1}$ is viewed as a $n\times n$ diagonal matrix that is made up of $[r_{t+1,1},...,r_{t+1,i},...,r_{t+1,n}]$. 

\begin{lemma}\label{lempre}
Under the assumptions in Theorem \ref{thm4}, we have for all $t\in [T]$, 
\begin{enumerate}[label=(\roman*)]
\item $\|\nabla f(\theta_t)\|_\infty\leq G_\infty$.
\item $\E[(\tilde F_t)^2]= F^2(\theta_t)=f(\theta_t)+c$.
\item 
$\E[\tilde F_t]\leq F(\theta_t)$. In particular, $\E[r_{1,i}]= \E[\tilde F_1]\leq F(\theta_1)$ for all $i\in[n]$.
\item 
$\sigma^2_g=\E[\|g_t-\nabla f(\theta_t)\|^2]\leq G^2_\infty$ and $\sigma^2_f=\E[\lvert f(\theta_t;\xi_t)- f(\theta_t)\rvert^2]\leq B^2.$
\item $\E[\lvert F(\theta_t)-\tilde F_t\rvert]\leq \frac{1}{2\sqrt{a}}\sigma_f$.
\item $\E[\|\nabla F(\theta_t)-\frac{g_t}{2\tilde F_t}\|^2]\leq \frac{G^2_\infty}{8a^3}\sigma^2_f+\frac{1}{2a}\sigma^2_g.$
\end{enumerate}
\end{lemma}
\begin{proof}
(i) By assumption $\|g_t\|_\infty\leq G_\infty$, we have
$$\|\nabla f(\theta_t)\|_\infty=\|\E[g_t]\|_\infty\leq\E[\|g_t\|_\infty]\leq G_\infty.$$
(ii) This follows from the unbiased sampling of 
$$
f(\theta_t)=\E_{\xi_t}[ f(\theta_t; \xi_t)].
$$
(iii) By Jensen's inequality, we have
$$\E[\tilde F_t] 
\leq\sqrt{\E[\tilde F_t^2]}=\sqrt{F(\theta_t)^2}=F(\theta_t).$$
(iv) By assumptions $\|g_t\|_\infty\leq G_\infty$ and $f(\theta_t;\xi_t)+c<B$, we have
$$
\sigma^2_g=\E[\|g_t-\nabla f(\theta_t)\|^2]
= \E[\|g_t\|^2] - \|\nabla f(\theta_t)\|^2\leq G^2_\infty,
$$
$$
\sigma^2_f=\E[\|f(\theta_t;\xi_t)-f(\theta_t)\|^2]
= \E[\|f(\theta_t;\xi_t)\|^2] - \| f(\theta_t)\|^2\leq B^2.
$$
(v) By the assumption $0<a\leq   f(\theta_t;\xi_t)+c=\tilde F_t^2$, we have
\begin{align*}
&\quad \E[\lvert F(\theta_t)-\tilde F_t\rvert]
\leq \E\Bigg[\bigg\lvert\frac{f(\theta_t)-f(\theta_t;\xi_t)}{F(\theta_t)+\tilde F_t}\bigg\rvert\Bigg] \leq \frac{1}{2\sqrt{a}}\E[\lvert f(\theta_t)-f(\theta_t;\xi_t)\rvert] 
\leq \frac{1}{2\sqrt{a}}\sigma_f.
\end{align*}
(vi) By the definition of $F(\theta)$, we have
\begin{align*}
\|\nabla F(\theta_t)-\frac{g_t}{2\tilde F_t}\|^2
&= \bigg\|\frac{\nabla f(\theta_t)}{2F(\theta_t)}-\frac{g_t}{2\tilde F_t}\bigg\|^2\\
&= \frac{1}{4}\bigg\|\frac{\nabla f(\theta_t)(\tilde F_t-F(\theta_t)) }{F(\theta_t)\tilde F_t}+\frac{\nabla f(\theta_t)-g_t}{\tilde F_t}\bigg\|^2\\
&\leq \frac{1}{2} \bigg\|\frac{\nabla f(\theta_t)(\tilde F_t-F(\theta_t)) }{F(\theta_t)\tilde F_t}\bigg\|^2 + \frac{1}{2}\bigg\|\frac{\nabla f(\theta_t)-g_t}{\tilde F_t}\bigg\|^2\\
&\leq \frac{G^2_\infty}{2a^{2}}\lvert\tilde F_t-F(\theta_t)\rvert^2+\frac{1}{2a}\|\nabla f(\theta_t)-g_t\|^2, 
\end{align*}
where both the gradient bound and the assumption that $0<a\leq f(\theta_t;\xi_t)+c=\tilde F^2_t$ are essentially used.  Take an expectation to get 
\begin{align*}
\E[\|\nabla F(\theta_t)-\frac{g_t}{2\tilde F_t}\|^2]\leq \frac{G^2_\infty}{2a^{2}}\E[\lvert\tilde F_t-F(\theta_t)\rvert^2]+\frac{1}{2a}\E[\|\nabla f(\theta_t)-g_t\|^2].
\end{align*}
Similar to the proof for ($iv$), we have
$$
\E[\lvert\tilde F_t-F(\theta_t)\rvert^2]\leq \frac{1}{4a}\sigma^2_f.
$$
This together with the variance assumption  for $g_t$ gives
\begin{equation*}
\E[\|\nabla F(\theta_t)-\frac{g_t}{2\tilde F_t}\|^2]\leq \frac{G^2_\infty}{8a^3}\sigma^2_f+\frac{1}{2a}\sigma^2_g.
\end{equation*}
\end{proof}

\begin{lemma}\label{lemrv}
For any $T\geq1$, we have
\begin{enumerate}[label=(\roman*)]
\item $\E\Big[\sum_{t=1}^{T} v_t^\top r_{t+1} v_t\Big] \leq \frac{n F(\theta_1)}{2\eta}$.
\item $\E\Big[\sum_{t=1}^{T} m_{t-1}^\top r_{t+1} m_{t-1}\Big]\leq \E\Big[\sum_{t=1}^{T} m_t^\top r_{t+1} m_t\Big] \leq \frac{2Bn F(\theta_1)}{\eta}$.
\item $\E\Big[\sum_{t=1}^{T}\|r_{t+1}m_t\|^2\Big] \leq \frac{2Bn F^2(\theta_1)}{\eta}$.
\item $\E\Big[\sum_{t=1}^{T} g_t^\top r_{t+1} g_t\Big] \leq \frac{8Bn F(\theta_1)}{(1-\beta)^2\eta}$.
\item $\E\Big[\sum_{t=1}^{T}\|r_{t+1}g_t\|^2\Big] \leq \frac{8Bn F^2(\theta_1)}{(1-\beta)^2\eta}$.
\end{enumerate}
\end{lemma}

\begin{proof}
From Algorithm \ref{alg} line 5, we have
$$
r_{t,i}-r_{t+1,i} = 2\eta r_{t+1,i}v^2_{t,i}.
$$
Taking summation over $t$ from $1$ to $T$ gives
$$
r_{1,i} - r_{T+1,i} = 2\eta\sum_{t=1}^{T} r_{t+1,i}v^2_{t,i}
\quad\Rightarrow\quad \sum_{t=1}^{T} r_{t+1,i}v^2_{t,i}\leq\frac{r_{1,i}}{2\eta}.
$$
From which we get
$$
\sum_{t=1}^{T}v_t^\top r_{t+1}v_t=\sum_{i=1}^{n}\sum_{t=1}^{T}r_{t+1,i}v^2_{t,i}\leq \frac{n\tilde F_1}{2\eta}.
$$
Taking expectation and using (iii) in Lemma \ref{lempre} gives (i). 

Recall that $m_t=2(1-\beta^t)\tilde F_tv_t$ and $\tilde F_t\leq\sqrt{B}$, we further get
$$
\sum_{t=1}^{T} m_t^\top r_{t+1} m_t
\leq 4B \sum_{t=1}^{T} v_t^\top r_{t+1}v_t
= \frac{2Bn\tilde F_1}{\eta}.
$$
Using $r_{t+1,i}\leq r_{t,i}$ and $m_{0,i}=0$, we also have
\begin{equation}\label{rm-1}
\sum_{i=1}^{n}\sum_{t=1}^{T}r_{t+1,i}m^2_{t-1,i}
\leq \sum_{i=1}^{n}\sum_{t=1}^{T}r_{t,i}m^2_{t-1,i}
= \sum_{i=1}^{n}\sum_{t=1}^{T-1}r_{t+1,i}m^2_{t,i}\leq \sum_{i=1}^{n}\sum_{t=1}^{T}r_{t+1,i}m^2_{t,i}.    
\end{equation}
Connecting the above two inequalities and taking expectation gives (ii). 

Using $r_{t+1,i}\leq r_{1,i}$, the above inequality further implies
\begin{align*}
\sum_{t=1}^{T}\|r_{t+1}m_t\|^2
&=\sum_{i=1}^n\sum_{t=1}^{T} r^2_{t+1, i}m_{t, i}^2  \leq \sum_{i=1}^n\sum_{t=1}^{T} r_{1,i} r_{t+1, i}m_{t, i}^2 \\
&=\bigg(\sum_{i=1}^n\sum_{t=1}^{T} r_{t, i}m_{t, i}^2\bigg)\tilde F_1
\leq 2Bn\tilde F^2_1/\eta.
\end{align*}
Taking expectation and using (ii) in Lemma \ref{lempre} gives (iii).

By $m_t=\beta m_{t-1}+(1-\beta)g_t$, we have
\begin{align*}
\sum_{t=1}^{T} g_t^\top r_{t+1} g_t 
&= \sum_{i=1}^{n}\sum_{t=1}^{T}r_{t+1,i}g^2_{t,i}
= \sum_{i=1}^{n}\sum_{t=1}^{T}r_{t+1,i}\bigg(\frac{1}{1-\beta}m_{t,i}-\frac{\beta}{1-\beta}m_{t-1,i}\bigg)^2\\
&\leq \frac{2}{(1-\beta)^2}\sum_{i=1}^{n}\sum_{t=1}^{T}r_{t+1,i}m^2_{t,i} + \frac{2\beta^2}{(1-\beta)^2}\sum_{i=1}^{n}\sum_{t=1}^{T}r_{t+1,i}m^2_{t-1,i}\\
&\leq \frac{2(1+\beta^2)}{(1-\beta)^2}\sum_{t=1}^{T} m_t^\top r_{t+1} m_t\leq \frac{8Bn\tilde F_1}{(1-\beta)^2\eta}.
\end{align*}
Here the third inequality is by $(a+b)^2\leq 2a^2+2b^2$; (\ref{rm-1}) and $0<\beta<1$ are used in the fourth inequality. Taking expectation and using (iii) in Lemma \ref{lempre} gives (iv).

Similar as the derivation for (ii), we have
\begin{align*}
\sum_{t=1}^{T}\|r_{t+1}g_t\|^2
\leq \bigg(\sum_{i=1}^n\sum_{t=1}^{T} r_{t, i}g_{t, i}^2\bigg)\tilde F_1
&\leq \frac{8Bn\tilde F^2_1}{(1-\beta)^2\eta}.
\end{align*}
Taking expectation and using (ii) in Lemma \ref{lempre} gives (v).
\end{proof}

First note that by (iv) in Lemma \ref{lempre}, $\max\{\sigma_g,\sigma_f\}\leq\max\{G_\infty,B\}$.\\
Recall that $F(\theta)=\sqrt{f(\theta)+c}$, then for any $x, y\in \{\theta_t\}_{t=0}^T$ we have 
\begin{align*}
\|\nabla F(x)-\nabla F(y)\|
&= \bigg\|\frac{\nabla f(x)}{2F(x)}-\frac{\nabla f(y)}{2F(y)}\bigg\|\\
&= \frac{1}{2}\bigg\|\frac{\nabla f(x)(F(y)-F(x))}{F(x)F(y)} + \frac{\nabla f(x)-\nabla f(y)}{F(y)}\bigg\|\\
&\leq \frac{G_\infty}{2(F(\theta^*))^2}|F(y)-F(x)| + \frac{1}{2F(\theta^*)}\|\nabla f(x)-\nabla f(y)\|.
\end{align*}
One may check that 
$$
\lvert F(y)-F(x)\rvert\leq \frac{G_\infty}{2F(\theta^*)}\|x-y\|.
$$
These together with the $L$-smoothness of $f$ lead to
\begin{equation*}
\|\nabla F(x)-\nabla F(y)\| \leq 
L_F \|x-y\|,
\end{equation*}
where 
$$
L_F=
\frac{1}{2\sqrt{f(\theta^*)+c}} \left( L+ \frac{G^2_\infty}{2(f(\theta^*)+c)}\right). 
$$
This confirms the $L_F$-smoothness of $F$, which yields 
\begin{align*}
 F(\theta_{t+1}) -  F(\theta_t)
& \leq\nabla F(\theta_t)^\top (\theta_{t+1}-\theta_t) +\frac{L_F}{2}\|\theta_{t+1}-\theta_t\|^2 \\
& = (\nabla F(\theta_t)-\frac{g_t}{2\tilde F_t})^\top (\theta_{t+1}-\theta_t) + (\frac{g_t}{2\tilde F_t}-\frac{1-\beta^t}{1-\beta}v_t)^\top (\theta_{t+1}-\theta_t)\\
&\quad +(\frac{1-\beta^t}{1-\beta}v_t)^\top (\theta_{t+1}-\theta_t) +\frac{L_F}{2}\|\theta_{t+1}-\theta_t\|^2. \\
\end{align*}
Summation of the above over $t$ from $1$ to $T$ and taken with the expectation gives 
\begin{equation}\label{ET1}
\E[F(\theta_{T+1})-F(\theta_{1})]\leq \sum_{i=1}^{4} S_i,    
\end{equation}
where 
\begin{align*}
&S_1=\E\Bigg[\sum_{t=1}^{T}\frac{1-\beta^t}{1-\beta}v_{t}^\top (\theta_{t+1}-\theta_t)\Bigg],\\
&S_2=\E\Bigg[\sum_{t=1}^{T}(\frac{g_t}{2\tilde F_t}-\frac{1-\beta^t}{1-\beta}v_t)^\top (\theta_{t+1}-\theta_t)\Bigg],\\
&S_3=\E\Bigg[\sum_{t=1}^{T}(\nabla F(\theta_t)-\frac{g_t}{2\tilde F_t})^\top (\theta_{t+1}-\theta_t)\Bigg],\\
&S_4=\E\Bigg[\sum_{t=1}^{T}\frac{L_F}{2}\|\theta_{t+1}-\theta_t\|^2\Bigg].
\end{align*}
Below we bound $S_1, S_2, S_3, S_4$ separately. To bound $S_1$, we first note that
\begin{align*}
r_{t+1,i}-r_{t,i}&=-2\eta r_{t+1,i}v^2_{t,i} = v_{t,i}(-2\eta r_{t+1,i}v_{t,i})=v_{t,i}(\theta_{t+1,i}-\theta_i) 
\end{align*}
from which we get
\begin{align*}
S_1
&= \E\Bigg[\sum_{t=1}^{T}\frac{1-\beta^t}{1-\beta}v_{t}^\top (\theta_{t+1}-\theta_t)\Bigg]\\
&= \E\Bigg[\sum_{i=1}^{n}\sum_{t=1}^{T} \frac{1-\beta^t}{1-\beta} (r_{t+1,i}-r_{t,i})\Bigg]\\ 
&\leq \E\Bigg[\sum_{i=1}^{n}\sum_{t=1}^{T} r_{t+1,i}-r_{t,i}\Bigg]\quad\text{(Since $r_{t+1,i}\leq r_{t,i}$)}\\
&= \sum_{i=1}^{n}\E[r_{T+1,i}]-n\E[\tilde F_1].
\end{align*}

For $S_2$, we have
\begin{align*}
S_2
&= \E\Bigg[\sum_{t=1}^{T}(\frac{g_t}{2\tilde F_t}-\frac{1-\beta^t}{1-\beta}v_t)^\top (\theta_{t+1}-\theta_t)\Bigg]\\
&= \E\Bigg[\sum_{i=1}^{n}\sum_{t=1}^{T}(-\frac{1}{2\tilde F_t}\frac{\beta}{1-\beta}m_{t-1,i})^\top(-2\eta r_{t+1,i}v_{t,i})\Bigg]\\
&\leq \frac{\beta\eta}{(1-\beta)\sqrt{a}}\E\Bigg[\left\lvert\sum_{i=1}^{n}\sum_{t=1}^{T}r_{t+1,i}m_{t-1,i}v_{t,i}\right\rvert\Bigg]\\
&\leq \frac{\beta\eta}{(1-\beta)\sqrt{a}}\E\Bigg[\bigg(\sum_{i=1}^{n}\sum_{t=1}^{T}r_{t+1,i}m^2_{t-1,i}\bigg)^{1/2}\bigg(\sum_{i=1}^{n}\sum_{t=1}^{T}r_{t+1,i}v^2_{t,i}\bigg)^{1/2}\Bigg]\\
&\leq \frac{\beta\sqrt{B}nF(\theta_1)}{(1-\beta)\sqrt{a}},
\end{align*}
where the fourth inequality is by the  Cauchy-Schwarz inequality, the last inequality is by Lemma \ref{lempre} (i) (ii).

For $S_3$, by the Cauchy-Schwarz inequality, 
we have
\begin{align*}
S_3
&= \E\Bigg[\sum_{t=1}^{T}(\nabla F(\theta_t)-\frac{g_t}{2\tilde F_t})^\top (\theta_{t+1}-\theta_t)\Bigg]\\
&\leq \E\Bigg[\sum_{t=1}^{T}\|\nabla F(\theta_t)-\frac{g_t}{2\tilde F_t})\| \|\theta_{t+1}-\theta_t)\|\Bigg]\\
&\leq \E\Bigg[\bigg(\sum_{t=1}^{T}\|\nabla F(\theta_t)-\frac{g_t}{2\tilde F_t})\|^2\bigg)^{1/2}\bigg(\sum_{t=1}^{T}\|\theta_{t+1}-\theta_t\|^2\bigg)^{1/2}\Bigg]\\
&\leq \Bigg(\E\Bigg[\sum_{t=1}^{T}\|\nabla F(\theta_t)-\frac{g_t}{2\tilde F_t})\|^2\Bigg]\Bigg)^{1/2}\Bigg(\E\Bigg[\sum_{t=1}^{T}\|\theta_{t+1}-\theta_t\|^2\Bigg]\Bigg)^{1/2}\\
&\leq F(\theta_1)\sqrt{\eta n T}\sqrt{\frac{G^2_\infty}{8a^3}\sigma^2_f+\frac{1}{2a}\sigma^2_g},
\end{align*}
where the last inequality is by (vi) in Lemma \ref{lempre} and (\ref{srev1+}) in Theorem \ref{thm1}.

For $S_4$, also by (\ref{srev1+}) in Theorem \ref{thm1}, we have
\begin{align*}
S_4
= \frac{L_F}{2}\E\Bigg[\sum_{t=1}^{T}\|\theta_{t+1}-\theta_t\|^2\Bigg]
\leq \frac{L_F\eta n F^2(\theta_1)}{2} .
\end{align*}

With the above bounds on $S_1, S_2, S_3, S_4$, (\ref{ET1})  can be rearranged as
\begin{align*}
&\quad F(\theta^*) - \frac{\beta\sqrt{B}nF(\theta_1)}{(1-\beta)\sqrt{a}} -F(\theta_1)\sqrt{\eta n T}\sqrt{\frac{G^2_\infty}{4a^3}\sigma^2_f+\frac{1}{a}\sigma^2_g}- \frac{L_F\eta n F^2(\theta_1)}{2}\\
&\leq  \sum_{i=1}^{n}\E[r_{T+1,i}] - n \E[\tilde F_1] +  F(\theta_1) \\
&\leq  \Big(\min_i \E[r_{T+1,i}]+(n-1)\E[\tilde F_1]\Big)- (n-1)\E[\tilde F_1] + \Big(F(\theta_1)-\E[\tilde F_1]\Big)\\
&\leq \min_i\E[r_{T+1,i}]+ \E[\lvert F(\theta_1)-\tilde F_1\rvert] \\
& \leq \min_i\E[r_{T+1,i}]+ \frac{1}{2\sqrt{a}}\sigma_f,
\end{align*}
where (iii) in Lemma \ref{lempre} was used. Hence,  
\begin{equation*}
\min_i\E[r_{T,i}]\geq \max\{F(\theta^*)-\eta D_1-\beta D_2-\sigma D_3,0\},   
\end{equation*}
where $\sigma=\max\{\sigma_f,\sigma_g\}$ and
\begin{align*}
&D_1 = \frac{L_F n F^2(\theta_1)}{2}, \quad D_2 =\frac{\sqrt{B}nF(\theta_1)}{(1-\beta)\sqrt{a}},\\
&D_3 = \frac{1}{2\sqrt{a}} + F(\theta_1)\sqrt{\eta n T}\sqrt{\frac{G^2_\infty}{4a^3}+\frac{1}{a}}.
\end{align*}
In the case $\sigma=0$, we obtain the stated estimate in Theorem \ref{thm4}.

\section{Proof of Theorem \ref{thm2}}\label{pf2}
The upper bound on $\sigma_g$ is given by (iv) in Lemma \ref{lempre}. Since $f$ is $L$-smooth, we have
\begin{align}\label{fL}
f(\theta_{t+1})\leq f(\theta_t)+\nabla f(\theta_t)^\top(\theta_{t+1}-\theta_t)+\frac{L}{2}\|\theta_{t+1}-\theta_t\|^2.
\end{align}
Denoting $\eta_t=\eta/\tilde F_t$, the second term in the RHS of (\ref{fL}) can be expressed as
\begin{align}\notag
&\quad\nabla f(\theta_t)^\top(\theta_{t+1}-\theta_t)\\\notag
&= \nabla f(\theta_t)^\top(-2\eta r_{t+1}v_{t})\\\notag
&= -\frac{1}{1-\beta^t}\nabla f(\theta_t)^\top \eta_t r_{t+1}m_{t}\quad\text{(since$\;m_t=2(1-\beta^t)\tilde F_t v_t$)}\\\label{T2}
&= -\frac{1}{1-\beta^t}\nabla f(\theta_t)^\top \eta_t r_{t+1}(\beta m_{t-1}+(1-\beta)g_t)\\\notag
&= -\frac{1-\beta}{1-\beta^t}\nabla f(\theta_t)^\top \eta_t r_{t+1}g_{t} - \frac{\beta}{1-\beta^t}\nabla f(\theta_t)^\top \eta_t r_{t+1}m_{t-1}\\\notag
&= -\frac{1-\beta}{1-\beta^t}\nabla f(\theta_t)^\top \eta_{t-1} r_{t}g_{t} + \frac{1-\beta}{1-\beta^t}\nabla f(\theta_t)^\top (\eta_{t-1} r_{t}-\eta_t r_{t+1}) g_{t}\\\notag
&\quad-\frac{\beta}{1-\beta^t}\nabla f(\theta_t)^\top \eta_t r_{t+1}m_{t-1}.
\end{align}    
We further bound the second term and third term in the RHS of (\ref{T2}), respectively. For the second term, we note that $\lvert\frac{1-\beta}{1-\beta^t}\rvert\leq 1$ and
\begin{align}\notag
&\quad\lvert\nabla f(\theta_t)^\top (\eta_{t-1}r_{t}-\eta_tr_{t+1}) g_{t}\rvert\\\notag
&= \lvert\nabla f(\theta_t)^\top \eta_{t-1}(r_{t}-r_{t+1}) g_{t}+ \nabla f(\theta_t)^\top (\eta_{t-1}-\eta_t)r_{t+1}g_{t}\rvert\\\notag
&= \lvert\nabla f(\theta_t)^\top \eta_{t-1}(r_{t}-r_{t+1}) g_{t}+ (\eta_{t-1}-\eta_t)g_t^\top r_{t+1}g_{t}\\\notag
&\quad+ (\eta_{t-1}-\eta_t)(\nabla f(\theta_t)-g_t)^\top r_{t+1}g_{t}\rvert\\\notag
&\leq \|\nabla f(\theta_t)\|_\infty \lvert\eta_{t-1}\rvert \|r_{t}-r_{t+1}\|_{1,1} \|g_{t}\|_\infty + \lvert\eta_{t-1}-\eta_t\rvert g_t^\top r_{t+1}g_{t}\\\notag
&\quad+ \lvert\eta_{t-1}-\eta_t\rvert\lvert(\nabla f(\theta_t)-g_t)^\top r_{t+1}g_{t}\rvert\\\notag
&\leq (\eta G^2_\infty/\sqrt{a})(\|r_t\|_{1,1}-\|r_{t+1}\|_{1,1})+(2\eta/\sqrt{a})g_t^\top r_{t+1}g_{t}\\\label{T21}
&\quad+ (2\eta/\sqrt{a})\lvert(\nabla f(\theta_t)-g_t)^\top r_{t+1}g_{t}\rvert.
\end{align}
The third inequality holds because for a positive diagonal matrix $A$, $x^\top Ay\leq\|x\|_\infty\|A\|_{1,1}\|y\|_\infty$, where $\|A\|_{1,1}=\sum_{i}a_{ii}$. The last inequality follows from the result $r_{t+1,i}\leq r_{t,i}$ for $i\in[n]$, the assumption $\|g_t\|_\infty\leq G_\infty$, $\tilde F_t\geq \sqrt{a}$, and (i) in Lemma (\ref{lempre}).

For the third term in the RHS of (\ref{T2}), we note that
$$
-\frac{\beta}{1-\beta^t} \nabla f(\theta_t)^\top \eta_t r_{t+1}m_{t-1} \leq \frac{\beta\eta}{(1-\beta)\sqrt{a}}\lvert\nabla f(\theta_t)^\top \eta_t r_{t+1}m_{t-1}\rvert,
$$
in which
\begin{align}\notag
&\quad \lvert\nabla f(\theta_t)^\top r_{t+1}m_{t-1}\rvert\\\notag
&= \lvert g_{t}^\top r_{t+1}m_{t-1}+(\nabla f(\theta_t)-g_t)^\top r_{t+1}m_{t-1}\rvert\\\label{T22}
&\leq \frac{1}{2}g_t^\top r_{t+1}g_t+\frac{1}{2}m_{t-1}^\top r_{t+1}m_{t-1}+ \lvert(\nabla f(\theta_t)-g_t)^\top r_{t+1}m_{t-1}\rvert,
\end{align}
where the last inequality is because for a positive diagonal matrix $A$, $x^\top Ay\leq \frac{1}{2}x^\top Ax+\frac{1}{2}y^\top Ay$.
Substituting (\ref{T21}) and (\ref{T22}) into (\ref{T2}), we get
\begin{equation}\label{T2s}
\begin{aligned}
\nabla f(\theta_t)^\top&(\theta_{t+1}-\theta_t)
\leq -\frac{1-\beta}{1-\beta^t}\nabla f(\theta_t)^\top \eta_{t-1}r_{t} g_{t} + \frac{\eta G^2_\infty}{\sqrt{a}}(\|r_{t}\|_{1,1}-\|r_{t+1}\|_{1,1})\\
&+\bigg(\frac{2\eta}{\sqrt{a}}+\frac{\beta\eta}{2(1-\beta)\sqrt{a}}\bigg)g_t^\top r_{t+1}g_t + \frac{\beta\eta}{2(1-\beta)\sqrt{a}}m_{t-1}^\top r_{t+1}m_{t-1}\\
&+\frac{2\eta}{\sqrt{a}}\lvert(\nabla f(\theta_t)-g_t)^\top r_{t+1}g_{t}\rvert+\frac{\beta\eta}{(1-\beta)\sqrt{a}}\lvert(\nabla f(\theta_t)-g_t)^\top r_{t+1}m_{t-1}\rvert.
\end{aligned}
\end{equation}
With (\ref{T2s}), we take an conditional expectation on (\ref{fL}) with respect to $(\theta)$ and rearrange to get
\begin{equation}\label{et}
\begin{aligned}
&\quad \frac{1-\beta}{1-\beta^t}\nabla f(\theta_t)^\top \eta_{t-1}r_{t}\nabla f(\theta_t)
=\E_{\xi_t}\bigg[\frac{1-\beta}{1-\beta^t}\nabla f(\theta_t)^\top \eta_{t-1}r_{t}g_t\bigg]\\
&\leq \E_{\xi_t}\Bigg[f(\theta_t)-f(\theta_{t+1})+ \frac{\eta G^2_\infty}{\sqrt{a}}(\|r_{t}\|_{1,1}-\|r_{t+1}\|_{1,1})\\
&\quad\quad+\bigg(\frac{2\eta}{\sqrt{a}}+\frac{\beta\eta}{2(1-\beta)\sqrt{a}}\bigg)g_t^\top r_{t+1}g_t + \frac{\beta\eta}{2(1-\beta)\sqrt{a}}m_{t-1}^\top r_{t+1}m_{t-1}\\
&\quad\quad+\frac{2\eta}{\sqrt{a}}\lvert(\nabla f(\theta_t)-g_t)^\top r_{t+1}g_{t}\rvert\\
&\quad\quad+\frac{\beta\eta}{(1-\beta)\sqrt{a}}\lvert(\nabla f(\theta_t)-g_t)^\top r_{t+1}m_{t-1}\rvert+\frac{L}{2}\|\theta_{t+1}-\theta_t\|^2\Bigg],
\end{aligned}
\end{equation}
where the assumption $\E_{\xi_t}[g_t]=\nabla f(\theta_t)$ is used in the first equality. Since $\xi_1,...,\xi_t$ are independent random variables, we set $\E=\E_{\xi_1}\E_{\xi_2}...\E_{\xi_T}$ and take a summation on (\ref{et}) over $t$ from 1 to $T$ to get
\begin{equation}\label{Esf}
\begin{aligned}
&\quad\E\Bigg[\sum_{t=1}^{T}\frac{1-\beta}{1-\beta^t}\nabla f(\theta_t)^\top \eta_{t-1}r_{t}\nabla f(\theta_t)\Bigg]\\
&\leq \E\Big[f(\theta_1)-f(\theta_{T+1})\Big] + \frac{\eta G^2_\infty}{\sqrt{a}}\E\Big[\|r_{1}\|_{1,1}-\|r_{T+1}\|_{1,1}\Big]\\
&\quad+ \bigg(\frac{2\eta}{\sqrt{a}}+\frac{\beta\eta}{2(1-\beta)\sqrt{a}}\bigg)\E\Bigg[\sum_{t=1}^{T}g_t^\top r_{t+1}g_t\Bigg] + \frac{\beta\eta}{2(1-\beta)\sqrt{a}}\E\Bigg[\sum_{t=1}^{T}m_{t-1}^\top r_{t}m_{t-1}\Bigg]\\
&\quad+\frac{2\eta}{\sqrt{a}}\E\Bigg[\sum_{t=1}^{T}\lvert(\nabla f(\theta_t)-g_t)^\top r_{t+1}g_{t}\rvert\Bigg]\\
&\quad+\frac{\beta\eta}{(1-\beta)\sqrt{a}}\E\Bigg[\sum_{t=1}^{T}\lvert(\nabla f(\theta_t)-g_t)^\top r_{t+1}m_{t-1}\rvert\Bigg]+\frac{L}{2}\E\Bigg[\sum_{t=1}^{T}\|\theta_{t+1}-\theta_{t}\|^2\Bigg].
\end{aligned}
\end{equation}
Below we bound each term in (\ref{Esf}) separately. By the Cauchy-Schwarz inequality, we get
\begin{align}\notag
&\quad \E\Bigg[\sum_{t=1}^{T}\lvert(\nabla f(\theta_t)-g_t)^\top r_{t+1}m_{t-1}\rvert\Bigg]\\\notag
&\leq \E\Bigg[\sum_{t=1}^{T}\|\nabla f(\theta_t)-g_t\| \|r_{t+1}m_{t-1}\|\Bigg]\\\notag
&\leq \E\Bigg[\bigg(\sum_{t=1}^{T}\|\nabla f(\theta_t)-g_t\|^2\bigg)^{1/2}\bigg(\sum_{t=1}^{T}\|r_{t+1}m_{t-1}\|^2\bigg)^{1/2}\Bigg]\\\notag
&\leq \Bigg(\E\Bigg[\sum_{t=1}^{T}\|\nabla f(\theta_t)-g_t\|^2\Bigg]\Bigg)^{1/2}\Bigg(\E\Bigg[\sum_{t=1}^{T}\| r_{t+1}m_{t-1}\|^2\Bigg]\Bigg)^{1/2}\\\label{vm}
&\leq \sqrt{2BnT/\eta}F(\theta_1)\sigma_g,
\end{align}
where Lemma \ref{lempre} (ii) and the bounded variance assumption were used. We replace $m_{t-1}$ in (\ref{vm}) by $g_t$ and use Lemma \ref{lempre} (v) to get
\begin{align}\notag 
&\quad\E\Bigg[\sum_{t=1}^{T}\lvert(\nabla f(\theta_t)-g_t)^\top r_{t+1}g_{t}\rvert\Bigg] \\\notag
&\leq \Bigg(\E\Bigg[\sum_{t=1}^{T}\|\nabla f(\theta_t)-g_t\|^2\Bigg]\Bigg)^{1/2}\Bigg(\E\Bigg[\sum_{t=1}^{T}\| r_{t+1}g_{t}\|^2\Bigg]\Bigg)^{1/2}\\ \label{vg}
& \leq \frac{2\sqrt{2BnT/\eta}F(\theta_1)\sigma_g}{1-\beta}.
\end{align}
By (\ref{srev1+}), the last term in (\ref{Esf}) is bounded above by 
\begin{equation}\label{dtheta^2}
\frac{L}{2}\E\left[\sum_{t=0}^\infty\|\theta_{t+1}-\theta_t\|^2 \right] \leq  \frac{L\eta n}{2}F^2(\theta_1).
\end{equation}
Substituting Lemma \ref{lempre} (i) (iii), (\ref{dtheta^2}), (\ref{vg}),  (\ref{vm}) into (\ref{Esf}) to get
\begin{equation}\label{rgb}
\begin{aligned}
\E\Bigg[\sum_{t=1}^{T}\frac{1-\beta}{1-\beta^t}&\nabla f(\theta_t)^\top\eta_{t-1} r_{t}\nabla f(\theta_t)\Bigg]
\leq (f(\theta_1)-f^*)+\frac{\eta G^2_\infty}{\sqrt{a}}nF(\theta_1)\\
&+\bigg(\frac{2}{\sqrt{a}}+\frac{\beta}{2(1-\beta)\sqrt{a}}\bigg)\frac{8BnF(\theta_1)}{(1-\beta)^2}+\frac{\beta BnF(\theta_1)}{(1-\beta)\sqrt{a}}\\
&+\frac{(4+\beta)\sqrt{2B\eta}}{(1-\beta)\sqrt{a}}F(\theta_1)\sqrt{nT}\sigma_g+\frac{L\eta n}{2}F^2(\theta_1).
\end{aligned}    
\end{equation}
Note that the left hand side is bounded from below by 
$$
(1-\beta)\frac{\eta}{\sqrt{B}}  \E\Bigg[\min_ir_{T,i}\sum_{t=1}^{T}\|\nabla f(\theta_t)\|^2\Bigg],
$$ 
where we used  $\lvert\frac{1-\beta}{1-\beta^t}\rvert\geq 1-\beta$ and $\eta_t\geq \eta/\sqrt{B}$. Thus we have 

\begin{align}\label{rf}
\E\Bigg[\min_ir_{T,i}\sum_{t=1}^{T}\|\nabla f(\theta_t)\|^2\Bigg]
\leq \frac{C_1+C_2n+C_3\sigma_g \sqrt{ nT}}{\eta },
\end{align}
where
\begin{align*}
C_1 &= \frac{(f(\theta_1)-f^*)\sqrt{B}}{1-\beta},\\
C_2 &=  \frac{\sqrt{B}\eta G^2_\infty F(\theta_1)}{(1-\beta)\sqrt{a}}
+\bigg(\frac{2}{\sqrt{a}}+\frac{\beta}{2(1-\beta)\sqrt{a}}\bigg)\frac{8B^{3/2}F(\theta_1)}{(1-\beta)^3}\\
&\quad +\frac{\beta B^{3/2}F(\theta_1)}{(1-\beta)^2\sqrt{a}}+\frac{\sqrt{B}L\eta }{2(1-\beta)^2}F^2(\theta_1),\\
C_3 &= \frac{(4+\beta)B\sqrt{2\eta}}{(1-\beta)\sqrt{a}}F(\theta_1).
\end{align*}
By the H\"{o}lder inequality, we have for any $\alpha \in (0, 1)$,
$$
\E[X^\alpha]\leq \E[XY]^\alpha \E[Y^{-\alpha/(1-\alpha)}]^{1-\alpha}.
$$
Take $X=\Delta:=\sum_{t=1}^{T}\|\nabla f(\theta_t)\|^2, Y=\min_ir_{T,i}$, we obtain 
$$
\E[\Delta^\alpha]\leq \E[\min_ir_{T,i}\Delta ]^\alpha \E[(\min_ir_{T,i})^{-\alpha/(1-\alpha)}]^{1-\alpha}.
$$
Using (\ref{rf}), and lower bounding $\E[\Delta^\alpha]$ by $T^{\alpha}\E[\min_{1\leq t\leq T}\|\nabla f(\theta_t)\|^{2\alpha}]$, we obtain 
$$
\E\Bigg[\min_{1\leq t \leq T}\|\nabla f(\theta_t)\|^{2\alpha} \Bigg]
\leq \left( \frac{C_1+C_2n+C_3\sigma_g \sqrt{ nT}}{\eta T } \right)^\alpha 
\E[(\min_ir_{T,i})^{-\alpha/(1-\alpha)}]^{1-\alpha}.
$$
This by taking $\alpha=1-\epsilon$ yields the stated bound.

\section{Proof of Theorem \ref{thm3}}
Using the same argument as for (iv) in Lemma \ref{lemrv}, we have
$$
\sum_{i=1}^{n}\sum_{t=1}^{T}r_{t+1,i}g^2_{t,i}\leq\frac{8Bn\sqrt{f_1(\theta_1)+c}}{(1-\beta)^2\eta}.
$$
With this estimate and the convexity of $f_t$, the regret can be bounded by
\begin{align*}\label{rt}
R(T) &= \sum_{t=1}^{T}f_t(\theta_t)-f_t(\theta^*)
\leq \sum_{t=1}^{T} g_t^\top (\theta_t-\theta^*)\\
&\leq \sum_{i=1}^{n}\sum_{t=1}^{T}\lvert g_{t,i}\rvert\sqrt{r_{t+1,i}} \frac{\lvert\theta_{t,i}-\theta^*_i\rvert}{\sqrt{r_{t+1,i}}}\\
&\leq \left(\sum_{i=1}^{n}\sum_{t=1}^{T}
r_{t+1, i} g_{t, i}^2 \right)^{1/2}\left(\sum_{i=1}^{n}\sum_{t=1}^{T} \frac{\lvert\theta_{t,i}-\theta^*_i\rvert^2}{r_{t+1,i}}\right)^{1/2}\\
&\leq \frac{2D_\infty\sqrt{2B}}{1-\beta}(f_1(\theta_1)+c)^{1/4}\sqrt{nT/\eta}\left(\sum_{i=1}^{n} \frac{1}{r_{T+1,i}}\right)^{1/2},
\end{align*}
where the fourth inequality is by the Cauchy-Schwarz inequality, and the assumption $\|x-y\|_\infty\leq D_\infty$ for all $x,y\in\mathcal{F}$ is used in the last inequality.

\medskip

\bibliographystyle{amsplain}
\bibliography{ref}

\end{document}